\documentclass{article}
\usepackage[utf8]{inputenc}

\usepackage[T1]{fontenc}    % use 8-bit T1 fonts
\usepackage{hyperref}       % hyperlinks
\usepackage{url}            % simple URL typesetting
\usepackage{booktabs}       % professional-quality tables
\usepackage{amsfonts}       % blackboard math symbols
\usepackage{nicefrac}       % compact symbols for 1/2, etc.
\usepackage{microtype}      % microtypography
\usepackage{xcolor}         % colors

\usepackage{stmaryrd}

\newcommand{\sE}[0]{\mathscr{E}}
\newcommand{\E}[0]{\mathbb{E}}

\newcommand{\N}[0]{\mathbb{N}}

\newcommand{\Pj}[0]{\mathbb{P}}

\newcommand{\R}[0]{\mathbb{R}}
\newcommand{\cR}[0]{\mathcal{R}}
\newcommand{\bS}[0]{\mathbb{S}}

\newcommand{\one}[0]{\mathbbm{1}}
\newcommand{\al}[0]{\alpha}
\newcommand{\be}[0]{\beta}
\newcommand{\ga}[0]{\gamma}
\newcommand{\Ga}[0]{\Gamma}
\newcommand{\de}[0]{\delta}
\newcommand{\De}[0]{\Delta}
\newcommand{\ep}[0]{\varepsilon}

\newcommand{\la}[0]{\lambda}

\newcommand{\ph}[0]{\varphi}

\newcommand{\Te}[0]{\Theta}

\newcommand{\Om}[0]{\Omega}
\newcommand{\si}[0]{\sigma}
\newcommand{\Si}[0]{\Sigma}

\newcommand{\nin}[0]{\not\in}

\newcommand{\sub}[0]{\subset}

\newcommand{\subeq}[0]{\subseteq}

\newcommand{\iy}[0]{\infty}
\newcommand{\rc}[1]{\frac{1}{#1}}
\newcommand{\prc}[1]{\pa{\rc{#1}}}

\newcommand{\fc}[2]{\frac{#1}{#2}}
\newcommand{\sfc}[2]{\sqrt{\frac{#1}{#2}}}
\newcommand{\pf}[2]{\pa{\frac{#1}{#2}}}%Shortcut for fraction with parentheses

\newcommand{\dd}[2]{\frac{d #1}{d #2}}

\newcommand{\ddd}[1]{\frac{d}{d #1}}

\newcommand{\nb}[0]{\nabla}

\newcommand{\dy}{\,dy}
\newcommand{\dx}{\,dx}

\newcommand{\ab}[1]{\left| {#1} \right|}
\newcommand{\an}[1]{\left\langle {#1}\right\rangle}
\newcommand{\ba}[1]{\left[ {#1} \right]}
\newcommand{\bc}[1]{\left\{ {#1} \right\}}

\newcommand{\pa}[1]{\left( {#1} \right)}

\newcommand{\ve}[1]{\left\Vert {#1}\right\Vert}

\newcommand{\set}[2]{\left\{{#1}:{#2}\right\}}

\newcommand{\ol}[1]{\overline{#1}}

\newcommand{\ub}[2]{\underbrace{#1}_{#2}}

\newcommand{\wt}[1]{\widetilde{#1}}
\newcommand{\wh}[1]{\widehat{#1}}

\newcommand{\Cov}{\operatorname{Cov}}

\newcommand{\diam}{\operatorname{diam}}

\newcommand{\Ent}{\operatorname{Ent}}

\newcommand{\KL}[0]{\operatorname{KL}}

\newcommand{\poly}{\operatorname{poly}}

\newcommand{\TV}[0]{\operatorname{TV}}

\providecommand{\cal}[1]{\mathcal{#1}}
\renewcommand{\cal}[1]{\mathcal{#1}}

\newcommand{\pull}[9]{
#1\ar@/_/[ddr]_{#2} \ar@{.>}[rd]^{#3} \ar@/^/[rrd]^{#4} & &\\
& #5\ar[r]^{#6}\ar[d]^{#8} &#7\ar[d]^{#9} \\}

\newcommand{\cmp}[9]{
\xymatrix{
#1 \ar[r]^{#4}{#5} \ar@/_2pc/[rr]^{#8}_{#9} & #2 \ar[r]^{#6}_{#7} & #3
}
}

\newcommand{\ha}[1]{\ar@{^(->}[#1]}
\newcommand{\ls}[1]{\ar@{-}[#1]}
\newcommand{\sj}[1]{\ar@{->>}[#1]}
\newcommand{\aq}[1]{\ar@{=}[#1]}
\newcommand{\acir}[1]{\ar@{}[#1]|-{\textstyle{\circlearrowright}}}
\newcommand{\acil}[1]{\ar@{}[#1]|-{\textstyle{\circlearrowleft}}}
\newcommand{\ard}[1]{\ar@{.>}[#1]}
\newcommand{\mt}[1]{\ar@{|->}[#1]}
\newcommand{\inm}[1]{\ar@{}[#1]|-{\in}}
\newcommand{\inr}{\ar@{}[d]|-{\rotatebox[origin=c]{-90}{$\in$}}}
\newcommand{\inl}{\ar@{}[u]|-{\rotatebox[origin=c]{90}{$\in$}}}

\newcommand{\sumo}[2]{\sum_{#1=1}^{#2}}
\newcommand{\sumz}[2]{\sum_{#1=0}^{#2}}

\newcommand{\beq}[1]{\begin{equation}\llabel{#1}}
\newcommand{\eeq}[0]{\end{equation}}
\newcommand{\bal}[0]{\begin{align*}}
\newcommand{\eal}[0]{\end{align*}}%this doesn't work; i don't know why
\newcommand{\ban}[0]{\begin{align}}
\newcommand{\ean}[0]{\end{align}}

\newcommand{\fixme}[1]{{\color{red}#1}}
\newcommand{\llabel}[1]{\label{#1}\text{\fixme{\tiny#1}}}

\newcommand{\arxiv}[1]{\url{http://www.arxiv.org/abs/#1}}

\DeclareFontFamily{U}{wncy}{}
    \DeclareFontShape{U}{wncy}{m}{n}{<->wncyr10}{}
    \DeclareSymbolFont{mcy}{U}{wncy}{m}{n}
    \DeclareMathSymbol{\Sh}{\mathord}{mcy}{"58} 

\usepackage{fullpage}

\usepackage{amsmath}
\usepackage{amssymb}
\usepackage{amsthm}
\usepackage{mathrsfs}
\usepackage{bbm}
\usepackage{cancel}
\usepackage{hyperref}
\usepackage{thm-restate}
\usepackage{xypic}
\usepackage{algorithm}
\usepackage{algorithmicx,algpseudocode}

\usepackage{color}

\usepackage[backend=biber,style=alphabetic,maxbibnames=99]{biblatex}
\newtheorem{thm}{Theorem}[section]
\newtheorem*{thm*}{Theorem}

\newtheorem*{clm*}{Claim}

\newtheorem*{conj*}{Conjecture}
\newtheorem{cor}[thm]{Corollary}
\newtheorem{lem}[thm]{Lemma}
\newtheorem*{lem*}{Lemma}

\newtheorem*{prb*}{Problem}

\newtheorem{asm}{Assumption}

\newtheorem*{ax*}{Axiom}
\newtheorem{df}[thm]{Definition}
\newtheorem*{df*}{Definition}

\newtheorem*{ex*}{Example}

\newtheorem*{pos*}{Postulate}

\newtheorem*{pr*}{Proposition}

\newtheorem*{qu*}{Question}

\newtheorem*{rem*}{Remark}

\newcommand{\citep}[1]{\cite{#1}}

\allowdisplaybreaks[1]
\title{Convergence of score-based generative modeling\\ 
for general data distributions}
\usepackage{authblk}
\author[1]{Holden Lee}
\author[2]{Jianfeng Lu}
\author[2]{Yixin Tan}
\affil[1]{Johns Hopkins University}
\affil[2]{Duke University}

\begin{document}

\newcommand{\CP}[0]{C_{\textup P}}
\newcommand{\CLS}[0]{C_{\textup{LS}}}
\newcommand{\tref}[1]{\text{\ref{#1}}}
\newcommand{\pdata}[0]{P_{\textup{data}}}
\newcommand{\ppr}[0]{p_{\textup{prior}}}
\newcommand{\etv}[0]{\ep_{\TV}}
\newcommand{\ew}[0]{\ep_{\textup{W}}}
\newcommand{\echi}[0]{\ep_{\chi}}
\newcommand{\eik}[0]{\ep_{\iy,T-t_k}}
\newcommand{\ek}[0]{\ep_{T-t_k}}

\definecolor{purple}{rgb}{0.5, 0.0, 0.5}
\newcommand{\hlnote}[1]{\textcolor{purple}{\textbf{HL:} #1}}

\newcommand{\yt}[1]{\textcolor{purple}{\textbf{YT:} #1}}

\newcommand{\jl}[1]{\textcolor{blue}{\textbf{JL:} #1}}

\renewcommand{\cR}[0]{\mathcal{R}}

\newcommand{\Bad}[0]{B}
\newcommand{\Good}[0]{A}
\newcommand{\smf}[0]{L}
%{\beta_f}
\newcommand{\smg}[0]{\beta_g} %DELETE
\newcommand{\sms}[0]{L_s}
%{\beta_s}
\newcommand{\score}[0]{s}

\newcommand{\Mkh}[0]{\varepsilon_{kh}}
\newcommand{\Gkht}[0]{G_{t_k,t}}
\newcommand{\FIqp}[0]{\mathscr{E}_{p_t}\pf{q_t}{p_t}}
\newcommand{\chiqpo}[0]{\chi^2(q_t||p_t)+1}
\newcommand{\Edz}[0]{\E\ba{\ve{z_t-z_{kh}}^2\psi_t(z_t)}}
\newcommand{\Edlnp}[0]{\E\ba{\ve{\nb \ln p_{kh}(z_t) - \nb \ln p_t(z_t)}^2\psi_t(z_t)}}
\newcommand{\KLqp}[0]{\KL(q_t||p_t)}

\newcommand{\Etkh}[0]{E}

\newcommand{\ncrs}[0]{n^{\textup{coarse}}}
\newcommand{\nfn}[0]{n^{\textup{fine}}}

\maketitle

\begin{abstract}
Score-based generative modeling (SGM) has grown to be a hugely successful method for learning to generate samples from complex data distributions such as that of images and audio. It is based on evolving an SDE that transforms white noise into a sample from the learned distribution, using estimates of the \emph{score function}, or gradient log-pdf. Previous convergence analyses for these methods have suffered either from strong assumptions on the data distribution or exponential dependencies, and hence fail to give efficient guarantees for the multimodal and non-smooth distributions that arise in practice and for which good empirical performance is observed.
We consider a popular kind of SGM---denoising diffusion models---and give polynomial convergence guarantees for general data distributions, with no assumptions related to functional inequalities or smoothness. Assuming $L^2$-accurate score estimates, we obtain Wasserstein distance guarantees for \emph{any} distribution of bounded support or sufficiently decaying tails, as well as TV guarantees for distributions with further smoothness assumptions.
\end{abstract}

\section{Introduction}
%\fixme{[General notes about SBM's]}

%As a method to learn and generate new samples from a probability distribution
Diffusion models have gained huge popularity in recent years in machine learning, as a method to learn and generate new samples from a data distribution. Score-based generative modeling (SGM), as a particular kind of diffusion model, uses learned score functions (gradients of the log-pdf) to transform white noise to the data distribution through following a stochatic differential equation. While SGM has achieved state-of-the-art performance for artificial image and audio generation \citep{song2019generative, dathathri2019plug, grathwohl2019your, song2020improved, song2020score, meng2021sdedit, song2021solving, song2021maximum,jing2022subspace},
including being a key component of text-to-image systems~\citep{ramesh2022hierarchical}, 
our theoretical understanding of these models is still nascent. 

In particular, basic questions on the convergence of the generated distribution to the data distribution remain unanswered.
Recent theoretical work on SGM has attempted to answer these questions \citep{ de2021diffusion,lee2022convergence,de2022convergence}, but they either suffer from exponential dependence on parameters or rely on strong assumptions on the data distribution such as %log-concavity
functional inequalities or smoothness, which are rarely satisfied in practical situations. 
For example, considering the hallmark application of generating images from text, we expect the distribution of images to be (a) multimodal, and hence not satisfying functional inequalities with reasonable constants, and (b) supported on lower-dimensional manifolds, and hence not smooth.
%Indeed, data distributions---such as those arising in the hallmark application of generating images from text---are often multimodal and hence do not satisfy functional inequalities with reasonable constants, and supported on lower-dimensional manifolds and hence not smooth. 
However, SGM still performs remarkably well in these settings. 
%This is in contrast to 
Indeed, this is one relative advantage to other approaches to generative modeling such as generative adversarial networks, which can struggle to learn multimodal distributions~\citep{arora2018gans}.
%requiring the data distribution to satisfy functional inequalities precludes efficient guarantees for multimodal distributions 
%Indeed, one of the most successful applications of diffusion model is conditional generation of images based on given text, for which the data distribution is multimodal and hence does not fit into the existing theory. 

In this work, we aim to develop theoretical convergence guarantees with polynomial complexity for SGM under minimal data assumptions.
%a more general class of distributions. 

\subsection{Problem setting}

%[Introduce DDPM and exponential integrator]
Given samples from a data distribution $\pdata$, the problem of generative modeling is to learn the distribution in a way that allows generation of new samples. 
A general framework for many score-based generative models is where noise is injected into %a data distribution $\pdata$ 
$\pdata$
via a forward SDE \citep{song2020score} 
\begin{align}\label{eq:forward_sde}
    d\wt x_t &= f(\wt x_t, t)\,dt + g(t)\, dw_t, \quad t\in [0,T],
\end{align}
where $\wt x_0\sim \wt P_0:=\pdata$. Let $\wt p_t$ denote the density of $\wt x_t$. Remarkably, $\wt x_t$ also satisfies a reverse-time SDE,
\begin{align}\label{eq:revSDE}
    d\wt x_t &= [f(\wt x_t, t) - g(t)^2 \nb \ln \wt p_t(\wt x_t)]\,dt + g(t) \,d\wt w_t, \quad t\in [0,T],
\end{align}
where $\wt w_t$ is a backward Brownian motion~\citep{anderson1982reverse}. Because the forward process transforms the data distribution to noise, the hope is to use the backwards process to transform noise into samples.

In practice, when we only have sample access to $\pdata$, the score function $\nb \ln \wt p_t$ is not available. A key mechanism behind SGM is that the score function is learnable from data, through empirically minimizing a de-noising objective evaluated at noisy samples $\wt x_t$~\citep{vincent2011connection}. The samples $\wt x_t$ are obtained by evolving the forward SDE starting from the data samples $\wt x_0$, and the optimization is done within an expressive function class such as neural networks.
%Of course, in practice the score function $\nb \ln \wt p_t$ is not available, and and has to be learned from empirical samples $\wt x_t$ from the forward SDE starting from the data samples. Denoising auto-encoders are often used for the score function estimate, which will 
% This will result in a guarantee that the $L^2$ error $\E_{\wt p_t}
% [\ve{\nb \ln \wt p_t(x) - s(x,t)}^2]$ is small \citep{vincent2011connection}.
If the score function is successfully approximated in this way, then the $L^2$-error $\E_{\wt p_t}
[\ve{\nb \ln \wt p_t(x) - s(x,t)}^2]$ will be small.
The natural question is then the following:
\begin{quote}
   Given $L^2$-error bounds of the score function, how close is the distribution generated by \eqref{eq:revSDE} (with score estimate $s(x,t)$ in place of $\nb \ln \wt p_t$, and appropriate discretization) to the data distribution $\pdata$?
\end{quote}
We note it is more realistic to consider $L^2$ rather than $L^\iy$-error, and this makes the analysis more challenging. 
Indeed, prior work on %the extensive literature on 
%durmus2017nonasymptotic, cheng2018convergence, cheng2018underdamped, dalalyan2017theoretical, dalalyan2019user, majka2020nonasymptotic,
Langevin Monte Carlo~\citep{ erdogdu2021convergence} and related sampling algorithms only apply when the score function is known exactly, or with suitable modification, known up to $L^\iy$-error. $L^2$-error has a genuinely different effect from $L^\iy$-error, as it can cause the stationary distribution for Langevin Monte Carlo to be arbitrarily diffferent~\citep{lee2022convergence}, necessitating a ``medium-time" analysis. 
%can be a starting point for the analysis, we note that a key difference in the SGM setting is that the score function is only assumed to be accurate in $L^2$, rather than $L^\iy$ error. 

In addition, we hope to obtain a result with as few structural assumptions as possible on $\pdata$, so that it can be useful in realistic scenarios where SGM is applied.

\subsection{Prior work on convergence guarantees}

%SGM's have recently attracted theoretical attention.
%The key problem we consider is convergence of the generated distribution to the data distribution under a $L^2$-accurate score function. 
We highlight two recent works which make progress on this problem.  
\cite{lee2022convergence} are the first to give polynomial convergence guarantees in TV distance under $L^2$-accurate score for a reasonable family of distributions. They introduce a framework to reduce the analysis under $L^2$-accurate score to $L^\iy$-accurate score. However, they rely on the data distribution satisfying smoothness conditions and a log-Sobolev inequality---a strong assumption which essentially limits the guarantees to unimodal distributions. %---and give convergence in TV distance. %work under the assumption of an $\ep$-accurate 

\cite{de2022convergence} instead make minimal data assumptions, giving convergence in Wasserstein distance for distributions with bounded support $\cal M$. In particular, this covers the case of distributions supported on lower-dimensional manifolds, where guarantees in TV distance are unattainable.
However, for general distributions, their guarantees have exponential dependence on the diameter of $\cal M$ and the inverse of the desired error ($\exp(O(\diam(\cal M)^2/\ep))$), and for smooth distributions, an improved, but still exponential dependence on the growth rate of the Hessian $\nb^2\ln \wt p_t$ as the noise approaches~0 ($\exp(\wt O(\Ga))$ for distributions with $\ve{\nb^2\ln \wt p_t}\le \Ga/\si_t^2$).
%However, their guarantees have dependence $\exp(O(\diam(\cal M)^2/\ep))$ for general distributions, and an improved, but still exponential dependence $\exp(\wt O(\Ga))$ for distributions with Hessian bounded by $\Ga/\si_t^2$.
%unclear...
%\hlnote{One strength of their analysis is that they can deal with a score estimate that's blowing up as $t\to 0$, by decomposing into tangent and normal directions to the manifold, while our bounds get worse as $t\to 0$.}

We note that other works also analyze the generalization error  of a learned score estimate \citep{block2020generative,de2022convergence}. This is an important question because without further assumptions, learning an $L^2$-accurate score estimate requires a number of samples exponential in the dimension. As this is beyond the scope of our paper, we assume that an $L^2$-accurate score estimate is obtainable.

%\hlnote{About concurrent and independent work...}\jl{added}

\subsection{Our contributions}

In this work, we analyze convergence in the most general setting of distributions of bounded support, as in~\cite{de2022convergence}. We give Wasserstein bounds for \emph{any} distribution of bounded support (or sufficiently decaying tails), and TV bounds for distributions under smoothness assumptions, that are polynomial in all parameters, and do not rely on the data distribution satisfying any functional inequality. This gives theoretical grounding to the empirical success of SGM on data distributions that are often multimodal and non-smooth.

We streamline the $\chi^2$-based analysis of~\cite{lee2022convergence}, with significant changes as to completely remove the use of functional inequalities. In particular, the biggest challenge---and our key improvement---is to bound a certain KL-divergence without reliance on a global functional inequality. For this, we prove a key lemma that %under mild smoothness and tail assumptions, distributions that are close in TV distance have score functions that are close in $L^2$. %which allows us to massage the distribution into a nicer one.
distributions which are close in $\chi^2$-divergence have score functions that are close in $L^2$ (which may be of independent interest), and then a structural result that the distributions arising from the diffusion process can be slightly modified as to satisfy the desired inequality, through decomposition into distributions that do satisfy a log-Sobolev inequality.
%(as mixtures of distributions that do satisfy a log-Sobolev inequality).

Upon finishing our paper, we learned of a concurrent and independent work \citep{chen2022sampling} which obtained theoretical guarantees for score-based generative modeling under similarly general assumptions on the data distribution. 
We note that although our bounds are obtained under similar assumptions %with \cite{chen2022sampling} 
(with our assumption of the score estimate accuracy slightly weaker than theirs), our proof techniques are quite different. Following the ``bad set'' idea from \cite{lee2022convergence}, we derived a change-of-measure inequality with Theorem~\ref{t:framework}, while the analysis in \cite{chen2022sampling} is based on the Girsanov approach.

\section{Main results}

To state %the problem more precisely
our results, we will consider a specific type of SGM called denoising diffusion probabilistic modeling (DDPM) \citep{ho2020denoising}, where in the forward SDE~\eqref{eq:forward_sde}, $f(x,t) = -\rc 2 g(t)^2x$ for some non-decreasing function $g$ to be chosen. The forward process is an Ornstein-Uhlenbeck process with time rescaling: $\wt x_t$ has the same distribution as 
\begin{align}
\label{e:OU}
& m_t \wt x_0 + \si_t z, 
\text{ where } \nonumber \\
& m_t = \exp\ba{-\rc 2\int_0^t g(s)^2\,ds}
, \, 
\si_t^2 = 1-\exp\ba{-\int_0^t g(s)^2\,ds},
\text{ and }z\sim N(0,I).
\end{align}
%\hlnote{$\si_t^2\le \max\{t,1\}$ when $g\equiv 1$}
Given an estimate score function $s(x,t)$ approximating $\nb \ln \wt p_t(x)$, we can simulate the reverse process (reparameterizing $t\mapsfrom T-t$ and denoting $p_{t}:=\wt p_{T-t}$) 
\begin{align}
\label{e:reverse}
    dx_t &= \rc 2 g(T-t)^2\pa{x_t + 2\nb \ln p_t(x_t)}\,dt + g(T-t)\,dw_t
\end{align}
with the exponential integrator discretization~\citep{zhang2022fast}, where $h_k = t_{k+1}-t_k$ and $\eta_{k+1}\sim N(0,I_d)$:
\begin{align}
\label{e:ei1}
    &z_{t_{k+1}} =
    z_{t_k} + \gamma_{1,k} (z_{t_k} + 2s(T-t_k, z_{t_k})) + \sqrt{\ga_{2,k}}\cdot  \eta_{k+1},\\
    &\text{where }
    \ga_{1,k} = \exp\ba{\rc2 G_{t_k,t_{k+1}}} - 1, \,
    \ga_{2,k} = \exp\ba{G_{t_k,t_{k+1}}}- 1, \text{ and}\,
    G_{t',t} := \int_{t'}^t g(T-s)^2 \,ds.
    \label{e:ei2}
\end{align}
We initialize $z_0$ with a prior distribution that approximates $p_0=\wt p_T$ for sufficiently large $T$: 
\begin{align}
\label{e:ppr}
z_0\sim q_0=\ppr :&=  N(0, %(1-e^{-\int_0^tg(s)^2\,ds})
\si_T^2 I_d) \approx N(0,I_d).
\end{align}
While we focus on DDPM, we note that the continuous process underlying DDPM is equivalent to that of score-matching Langevin diffusion (SMLD) under reparameterization in time and space (see \cite[\S C.2]{lee2022convergence}). We will further take $g\equiv 1$ for convenience in stating our results. %\hlnote{Alternatively, we can take the step size to be constant and $g$ to vary. In~\cite{song2020score}, $g(t) = \sqrt{\ol\be_{\min} + t(\ol\be_{\max} - \ol\be_{\min})}$. $g(t)=\sqrt t$ with constant step size corresponds to step size $h_k = \Te(1/t_k^2)$.}

Our goal is to obtain a
quantitative guarantee for the distance between the distribution $q_{t_K}$ for $z_{t_K}$ (for appropriate $t_K\approx T$) and $\pdata$, under a $L^2$-score error guarantee. In the following, we assume a sequence of discretization points $0=t_0<t_1<\cdots <t_K\le T$ has been chosen. 
%[Remark on equivalence with SMLD under reparameterization]

\begin{asm}[$L^2$ score error]\label{a:score}
%Let $\pdata$ be a given probability measure on $\R^d$ and let $p_t$ be the density [after time $t$]
%Let $p_t := \wt p_{T-t}$ for $t\in [0,T)$. 
For any $t\in \{T-t_0,\ldots, T-t_K\}$, the error in the score estimate is bounded in $L^2(\wt p_{t})$:
    \[
    \ve{\nb \ln \wt p_t-s(\cdot, t)}_{L^2(\wt p_t)}^2=
    \E_{\wt p_t}[\ve{\nb \ln \wt p_t(x) - s(x,t)}^2]\le \ep_t^2:= \fc{\ep_\si^2}{\si_t^4}.
    \]
    %where $\ep_t = \fc{\ep^2}{\si_t^2}$. 
%\hlnote{How should we let $\ep$ depend on $t$?}
\end{asm}
We note that the gradient $\nb \ln \wt p_t$ grows as $\rc{\si_t^2}$ as $t\to 0$, so this is a reasonable assumption, and quantitatively weaker than a uniform bound over $t$.
% \jl{move the sentence to a better position}
%A reasonable assumption is that the gradient $\nb \ln p_t$ grows as $\rc{\si_t^2}$ as $t\to 0$, in which case we can expect that $\ep_t^2 = \fc{\ep_0^2}{\si_t^4}$.

% \begin{asm}[Tail bound]
% \label{a:tail}
% $R:[0,1]\to [0,\iy)$ is a function such that $\pdata(B_{R(\ep)}(0))\ge 1-\ep$.
% \end{asm}
% By taking $R$ to be a constant function, this contains the assumption of bounded support as a special case.
\begin{asm}[Bounded support]
\label{a:bd}
$\pdata$ is supported on $B_R(0):=\set{x\in \R^d}{\ve{x}\le R}$.
\end{asm}
For simplicity, we assume bounded support when stating our main theorems, but note that our results generalize to distributions with sufficiently fast power decay.
In the application of image generation, pixel values are bounded, so Assumption~\ref{a:bd} is satisfied with $R$ typically on the order of $\sqrt d$.

These are the only assumptions we need to obtain a polynomial complexity guarantee. % We also introduce some refined assumptions that can give better dependences and further results. %\hlnote{Also need bound on $L_s$}
%We consider the following assumption, which is Assumption A.6 in \cite{de2022convergence}. 
We also consider the following stronger smoothness assumption, which is Assumption A.6 in \cite{de2022convergence} and will give better dependencies. Note that \cite[Theorem I.8]{de2022convergence} shows a (nonuniform) version of Assumption~\ref{a:smoothness} holds when $p_0$ is a smooth density on a convex submanifold. 
\begin{asm}\label{a:smoothness}
The following bound of the Hessian of the log-pdf holds for any $t>0$ and $x$:
\begin{align*}
    \ve{\nb^2 \ln p_t(x)}
    &\le \fc{C}{\si_t^2},
\end{align*}
for some constant $C>0$.
\end{asm}
%Note \cite[Theorem I.8]{de2022convergence} 
%shows that if $p_0$ is a smooth density with respect to the Hausdorff measure on a convex submanifold $\mathcal M\subeq \R^d$, then $\ve{\nb\ln p_t(x)} \le \fc{C_x}{\si_t^2}$ with a constant possibly depending on $x$, though Assumption~\ref{a:smoothness} requires a uniform-in-space bound. 
Finally, the following smoothness assumption on $\wt p_0$ will allow us to obtain TV guarantees.
\begin{asm}\label{a:smooth0}
$\pdata$ admits a density $\wt p_0\propto e^{-V(x)}$ where $V(x)$ is $L$-smooth. % on the interior of $B_R(0)$.
\end{asm}
We are now ready to state our main theorems.
\begin{algorithm}[h!]
\begin{algorithmic}
\State INPUT: Time $T$; discretization points $0=t_0<t_1<\cdots<t_N\le T$; score estimates $s(\cdot, T-t_k)$; radius $R$; function $g$ (default: $g\equiv 1$)
\State Draw $z_0\sim \ppr$ from the prior distribution $\ppr$ given by~\eqref{e:ppr}. %\fixme{... for SMLD, ... for DDPM}
\For{$k$ from $1$ to $N$} %\jl{I guess $m$ should be $k$ here?}
    \State Compute $z_{t_k}$ from $z_{t_{k-1}}$ using~\eqref{e:ei1}. %, with $g\equiv 1$.
\EndFor
\State Let $\wh z_{t_N} = z_{t_N}$ if $z_{t_N}\in B_R(0)$; otherwise, let $\wh z_{t_N}=0$.
%OUTPUT: If $z_{t_N}\in B_R(0)$, then return $z_{t_N}$; otherwise, return 0.
\end{algorithmic}
 \caption{DDPM with exponential integrator~\citep{song2020score,zhang2022fast}}
 \label{a:ddpm}
\end{algorithm}

%\fixme{Talk about ``surprising" nature of this question.}

\begin{thm}[Wasserstein+TV error for distributions with bounded support]
\label{t:main-tv-w2}
Suppose that Assumption~\ref{a:score} and~\ref{a:bd} hold with $R\ge \sqrt d$. Then there is a sequence of discretization points $0=t_0<t_1<\cdots <t_N<T$ with $N=O(\poly(d,R,1/\etv,1/\ew))$ such that if $\ep_\si = \wt O\pf{\etv^{6.5} \ew^5}{R^9d^{2.25}}$, then the distribution $q_{t_N}$ of the output $z_{t_N}$ of DDPM is $\etv$-close in TV distance to a distribution that is $\ep_W$ in $W_2$-distance from $\pdata$.
If in addition Assumption~\ref{a:smoothness} holds with $C\ge R^2$, it suffices for $\ep_\si = \wt O\pf{\etv^4}{C^2d}$ (note that the $\wt O(\cdot)$ hides logarithmic dependence on $\ew$).
\end{thm}

This result is perhaps surprising at first glance, as it is well known that for sampling algorithms such as Langevin Monte Carlo, structural assumptions on the target distribution---such as a log-Sobolev inequality---are required to obtain similar theoretical guarantees, even with the knowledge of the exact score function. The key reason that we can do better is that we utilize a \emph{sequence} of score functions $s_t$ along the reverse SDE, which is not available in standard sampling settings. Moreover, we choose $T$ large enough so that $q_0=\ppr$ is close to $p_0$, and it suffices to track the evolution of the true process~\eqref{eq:revSDE}, that is, maintain rather than decrease the error.
To some extent, this result shows the power of DDPM and other reverse SDE-based methods compared with generative modeling based on standard Langevin Monte Carlo.

A statement with more precise dependencies, and which works for unbounded distributions with sufficiently decaying tails, can be found as Theorem~\ref{t:ddpm-l2}. 
We note that under the Hessian bound (Assumption~\ref{a:smoothness}), up to logarithmic factors, the same score error bound suffices to obtain a fixed TV distance to a distribution arbitrarily close in $W_2$ distance.
By truncating the resulting distribution, we can also obtain purely Wasserstein error bounds. 
\begin{thm}[Wasserstein error for distributions with bounded support]
\label{t:main-w2}
In the same setting as Theorem~\ref{t:main-tv-w2}, consider the distribution $\wh q_{t_N}$ of the truncated output $\wh x_{t_N}$ of DDPM. 
If Assumptions~\ref{a:score} and~\ref{a:bd} hold with $R\ge \sqrt d$ and $\ep_\si  =\wt O\pf{\ew^{18}}{R^{22}d^{2.25}}$, then with appropriate (polynomial) choice of parameters,
$W_2(\wh q_{t_K}, \pdata)\le \ew$. 
If in addition Assumption~\ref{a:smoothness} holds with $C\ge R^2$, then $ \ep_\si = \wt O\pf{\ew^8}{C^2R^8 d}$ suffices.
% There is a sequence of discretization points $0=t_0<t_1<\cdots <t_N<T$ with $N=O(\poly(d,R,1/\ep_W))$ such that if Assumption~\ref{a:score} holds with score error %$\ep_0 = \wt O\pf{\ep_W^8}{R^{10}d^{4/3}}$ 
% $\ep_\si = \wt O\pf{\ep_W^{18}}{R^{22}d^{2.25}}$ 
% and Assumption~\ref{a:bd} holds with radius $R$ ($R\ge \sqrt d$), %$\ep_t = \Om(1/\poly(d,R,1/\ep_W))$
% then the output $\wh q_{t_K}$ of DDPM (Algorithm~\ref{a:ddpm}) satisfies
% \[
% W_2(\wh q_{t_K}, \pdata)\le \ep_W.
% \]
% If in addition Assumption~\ref{a:smoothness} holds with $C\ge R^2$, then it suffices to have $L^2$ score error $\wt O\pf{\ep_W^8}{C^2R^8 d}$.
\end{thm}

% \hlnote{Alternatively, can state it the following way (with smaller exponents).
% If $\ep_\si = \wt O\pf{\etv^{6.5} \ep_W^5}{R^9d^{2.25}}$, then the output $q_{t_K}$ of DDPM is $\etv$ close to a distribution that is $\ep_W$ in Wasserstein-2 distance from $\pdata$.
% Under Assumption~\ref{a:smoothness}, it suffices for $\ep_\si = \wt O\pf{\etv^4}{C^2d}$ (with only logarithmic dependence on $\ep_W$!)
% }

%\fixme{smoothness/tail assumptions}
%Note that 
With an extra assumption on the smoothness of $P_{\text{data}}$, we can also obtain purely TV error bounds:
\begin{thm}[TV error for distributions under smoothness assumption]\label{t:tv-tail}
%In the same setting of Theorem~\ref{t:main-tv-w2}, if
Suppose that Assumptions~\ref{a:score}  and~\ref{a:smooth0} hold, $\pdata$ is subexponential (with a fixed constant), and denote $R= \max\bc{\sqrt d, \E_{P_{\text{data}}}\ve{X}}$. Then there is a sequence of discretization points $0=t_0<t_1<\cdots <t_N<T$ with $N=O(\poly(d,R,1/\etv))$ such that if $\ep_\si  =\wt O\pf{\etv^{11.5}}{R^{14}d^{2.25}L^5}$, then
$\TV( q_{t_N}, \pdata)\le \ep_{\TV}$. 
If in addition Assumption~\ref{a:smoothness} holds with $C\ge R^2$, then $ \ep_\si = \wt O\pf{\etv^4}{C^2 d}$ suffices. %(note that the $\wt O(\cdot)$ hides logarithmic dependence on $\etv$ and $R$).
\end{thm}
A more precise statement %version of this theorem 
can be found as Theorem~\ref{t:ddpm-l2-TV}, which also works %for unbounded distributions with proper decaying tails. 
more generally with sufficient tail decay. We note that this result can be derived directly by combining Theorem~\ref{t:ddpm-l2} and a TV error bound between $P_{\text{data}}$ and $p_{t_N}$ %, which is presented in 
(Lemma~\ref{l:bdd_spt}) depending on the smoothness of $P_{\text{data}}$.

\section{Proof overview}

Our proof uses the framework by~\cite{lee2022convergence} to convert guarantees under $L^\iy$-accurate score function to under $L^2$-accurate score function.
For the analysis under $L^\iy$-accurate score function, we interpolate the discrete process with estimated score, $z_t\sim q_t$, and derive a differential inequality
\begin{align*}
%\label{e:d-chi2}
    \ddd t \chi^2(q_t||p_t) &= -g(T-t)^2 \FIqp %cE_{p_t}\pf{q_t}{p_t} 
    + 
    2 \E\ba{\an{g(T-t)^2(s(z_{t_k}, T-t_k) - \nb \ln %\tilde p_{T-t}(z_t))
    p_t(z_t), \nb \fc{q_t(x)}{p_t(x)}}}.
\end{align*}
We bound resulting error terms, making ample use of the Donsker-Varadhan variational principle to convert expectations to be under $p_t$. Under small enough step sizes, this shows that $\chi^2(q_t||p_t)$ grows slowly (Theorem~\ref{t:ddpm-liy}), which suffices as $\chi^2$-divergence decays exponentially in the forward process.
%, which is sufficient as $\chi^2(q_0||p_0)$ can be made 

The most challenging error term to deal with is the KL divergence term $\KL(q_t\psi_t||p_t)$.
Our main innovation over the analysis of~\cite{lee2022convergence} is bounding this term without a global log-Sobolev inequality for $p_t$. We note that it suffices for $p_t$ to be a mixture of distributions each satisfying a log-Sobolev inequality, with the logarithm of the minimum mixture weight bounded below, and in Lemma~\ref{l:KL-covering}, we show that we can decompose any distributed of bounded support in this manner if we move a small amount of its mass. 

In Section~\ref{s:perturb}, we show that this does not significantly affect the estimate of the score function, by interpreting the score function as solving a Bayesian inference problem: that of de-noising a noised data point. More precisely, we show in Lemma~\ref{lem:mismatched_prior} that the difference between the score functions of two different distributions can be bounded in $L^2$ in terms of their $\chi^2$-divergence, which may be of independent interest.

Finally, we reduce from the $L^2$ to $L^\iy$ setting by bounding the probabilities of hitting a bad set where the score error is large, and carefully choose parameters (Section~\ref{s:l2}). 
This gives a TV error bound to $\wt p_\de$---the forward distribution at small positive time. Finally, we can bound the Wasserstein distance of $\wt p_\de$ to $\wt P_0$ (in the general case) or the TV distance (under additional smoothness of $\wt P_0$.)

In Section~\ref{s:hess-whp} we show that the Hessian is always bounded by $O\pf{d}{\si_t^2}$ with high probability (cf. Assumption~\ref{a:smoothness}). We speculate that a high-probability rather than uniform bound on the Hessian (as in Lemma~\ref{l:Hess-bd}) can be used to obtain better dependencies, and leave this as an open problem.

%though we leave as an open problem how to use a high-probability rather than uniform bound to obtain better dependencies. \hlnote{possibly move this statement}

\subsection*{Notation and definitions}

%For a probability density $p$, we will denote the distribution by $p$ as well.

We let $\wt p_t$ denote the density of $\wt x_t$ under the forward process~\eqref{eq:forward_sde}. Note that $x_0\sim \wt P_0$ may not admit a density, but $\wt x_t$ will for $t>0$. For the reverse process, we use the notation $p_t = \wt p_{T-t}$, $x_t = \wt x_{T-t}$.
We defined $m_t$ and $\si_t$ in~\eqref{e:OU},
\begin{align*}
m_t &= \exp\ba{-\rc 2\int_0^t g(s)^2\,ds}
, &
\si_t^2 &= 1-\exp\ba{-\int_0^t g(s)^2\,ds},
\end{align*}
and note that $\wt p_t = (M_{m_t\sharp} \wt P_0)*\ph_{\si_t^2}$, where $M_m(x)=mx$ denotes multiplication by $m$, $F_\sharp P$ denotes the pushforward of the measure $P$ by $F$, and $\ph_{\si^2}$ is the density of $N(0,\si^2I_d)$. When $g\equiv 1$, we note the bound $\si_t^2 \le \min\{1,t\}$ and $\si_t^2 =\Te(\min\{1,t\})$. 

We will let $z_t$ denote the (interpolated) discrete process (see~\eqref{e:cts-time}) and let $q_t$ be the density of $z_t$.
We define
\begin{align}
\label{e:psi}
    \phi_t(x)& = \fc{q_t(x)}{p_t(x)},&
    \psi_t(x) &= \fc{\phi_t(x)}{\E_{p_t}\phi_t^2},
\end{align}
and note that $q_t\psi_t$ is a probability density. We defined $G_{t',t} = \int_{t'}^t g(T-s)^2 \,ds$ in~\eqref{e:ei2}.

We denote the estimated score function by either $s(x,t)$ and $s_t(x)$ interchangeably.

A random variable $X$ is subgaussian with constant $C$ if 
\begin{align*}
    C=\ve{X}_{\psi_2} :&=\inf\set{t>0}{\E\exp(X^2/t^2)\le 2} <\iy.
\end{align*}
A $\R^d$-valued random variable $X$ is subgaussian with constant $C$ if for all $v\in \bS^{d-1}$, $\an{X,v}$ is subgaussian. We also define 
\begin{align*}
    \ve{X}_{2,\psi_2}:&=\ve{\ve{X}_2}_{\psi_2}. 
\end{align*}

Given a probability measure $P$ on $\R^d$ with density $p$, the associated Dirichlet form is $\sE_p(f,g):=\int_{\R^d} \an{\nb f, \nb g}\, P(dx) = \int_{\R^d} \an{\nb f, \nb g} p(x)\,dx$; denote $\sE_p(f)=\sE_p(f,f)$. %If $q$ is another probability measure, the quantity $\sE_p(q/p)$ is known as the Fisher information of $q$ with respect to $p$. 
we say that a log-Sobolev inequality (LSI) holds with constant $\CLS$ if for any probability density $q$,
\begin{align}
    \KL(q||p)&\leq %\fc\lambda2 J_\mu(\nu) :&= 
    \fc{\CLS}2 \sE_p\pa{\fc qp, \ln \fc qp} =
    \fc{\CLS}2\int_{\R^d}\ve{\nb\ln\fc{q(x)}{p(x)}}^2q(x)\dx.%\tag{LSI}
\end{align}
Note $\int_{\R^d}\ve{\nb\ln\fc{q(x)}{p(x)}}^2q(x)\dx$ is also known as the Fisher information of $q$ with respect to $p$. 
Alternatively, defining the entropy by $\Ent_p(f) = \E_p f(x)\ln f(x) - \E_p f(x) \ln \E_p f(x)$, for any $f\ge 0$,
\begin{align}
    \Ent_p(f)& \le  \fc{\CLS}2 \sE_p\pa{f, \ln f} =
    \fc{\CLS}2\int_{\R^d}\ve{\nb\ln f(x)}^2 f(x) p(x)dx.%\tag{LSI}
\end{align}
%\section{Exponential integrator}

\section{DDPM with $L^\iy$-accurate score estimate}

We consider the error between the exact backwards SDE~\eqref{e:reverse}
and the exponential integrator with estimated score~\eqref{e:ei1}. 
In this section, we bound the error assuming that the score estimate $s$ is accurate in $L^\iy$.
\begin{asm}[$L^\iy$ score error]\label{a:score-liy}
%Let $\pdata$ be a given probability measure on $\R^d$ and let $p_t$ be the density [after time $t$]
%Let $p_t := \wt p_{T-t}$ for $t\in [0,T)$. 
For any $t\in \{T-t_0,\ldots, T-t_K\}$, the error in the score estimate is bounded: %in $L^2(\wt p_{t})$:
    \begin{align}
    \label{e:score-liy}
    \ve{\nb \ln \wt p_t-s(\cdot, t)}_{\iy}=
    \sup_{x\in \R^d} \ve{\nb \ln \wt p_t(x) - s(x , t)} 
    \le \ep_{\iy,t}^2
    \end{align}
    for some non-decreasing function $\ep_{\iy,t}^2$.
    %where $\ep_t = \fc{\ep^2}{\si_t^2}$. 
%\hlnote{How should we let $\ep$ depend on $t$?}
%\hlnote{Change notation to $\ep_{\iy,t}$ to avoid confusion with $L^2$ error.}
\end{asm}
In Section~\ref{s:l2}, we will relax this condition to score function being accurate in $L^2$.

% Using the exponential integrator gives the following discretization step, where $h_k = t_{k+1}-t_k$ and $\eta_{k+1}\sim N(0,I_d)$:
% \begin{align}
%     z_{t_{k+1}} &=
%     z_{t_k} + \gamma_{1,k} (z_{t_k} + 2s(T-t_k, z_{t_k})) + \sqrt{\ga_{2,k}}\cdot  \eta_{k+1},\\
%     \text{where }
%     \ga_{1,k} &= \exp\ba{\rc2 G_{t_k,t_{k+1}}} - 1\\
%     \ga_{2,k} &= \exp\ba{G_{t_k,t_{k+1}}}- 1\\
%     G_{t',t} :&= \int_{t'}^t g(T-s)^2 \,ds.
% \end{align}
First, we construct the following continuous-time process which interpolates the discrete-time process~\eqref{e:ei1}, for $t\in [t_k,t_{k+1}]$:
\begin{align}\label{e:cts-time}
    dz_t &= g(T-t)^2 \pa{\rc 2 z_t + s(z_{t_k}, T-t_k)} \,dt + g(T-t)\,dw_t.
\end{align}
Integration gives that
\begin{align}
    z_t - z_{t_k}
    &= 
    \pa{\exp\pa{\rc2 G_{t_k,t}}-1}(z_{t_k} + 2s(z_{t_k},T-t_k)) \nonumber \\
   \label{e:int-interp}
 & \qquad 
    + \int_{t_k}^t \exp\pa{\rc 2 \int_{t_k}^{t'} g(T-t'')^2\,dt''} g(t')\,  dw_{t'},
\end{align}
where $G_{t',t}$ is defined in~\eqref{e:ei2}.

Letting $q_t$ be the distribution of $z_t$ and $p_t$ be the distribution of $x_t$, we have by \cite[Lemma A.2]{lee2022convergence} that
\begin{align}
\label{e:d-chi2}
    \ddd t \chi^2(q_t||p_t) &= -g(T-t)^2 \FIqp %cE_{p_t}\pf{q_t}{p_t} 
    + 
    2 \E\ba{\an{g(T-t)^2(s(z_{t_k}, T-t_k) - \nb \ln %\tilde p_{T-t}(z_t))
    p_t(z_t), \nb \fc{q_t(x)}{p_t(x)}}}.
\end{align}
(Note that in our case, $\wh f$ also depends on $z_t$ rather than just $z_{t_k}$, but this does not change the calculation.) Define $\phi_t, \psi_t$ as in~\eqref{e:psi}:
$ \phi_t(x) = \fc{q_t(x)}{p_t(x)}$, $\psi_t(x) = \fc{\phi_t(x)}{\E_{p_t}\phi_t^2}$.
% We define
% \begin{align}
% \label{e:psi}
%     \phi_t(x)& = \fc{q_t(x)}{p_t(x)},&
%     \psi_t(x) &= \fc{\phi_t(x)}{\E_{p_t}\phi_t^2},
% \end{align}
% and note that $q_t\psi_t$ is a probability density.

%We will need the following lemma.
%Finally, we will make good use of the following lemma to bound the second term in Lemma~\ref{l:d-chi2}. 
To bound~\eqref{e:d-chi2}, we use the following lemma.
\begin{lem}[cf. {\cite[Lemma 1]{erdogdu2021convergence}, \cite[Lemma A.3]{lee2022convergence}}]
\label{l:inp-young}
%Let $\phi_t(x) = \fc{q_t(x)}{p_t(x)}$ and $\psi_t(x) = \phi_t(x)/\E_{p_t}\phi_t^2$. 
For any $C>0$ and any $\R^d$-valued random variable $u$, we have
\begin{align*}
\E\ba{\an{u ,\nb\fc{q_t(z_t)}{p_t(z_t)}}}
%\le \E\ba{\ve{u}\ve{\nb\fc{q_t(z_t)}{p_t(z_t)}}}
   & \leq 
   C\cdot %\E_{p_t}\phi_t^2
   (\chi^2(q_t||p_t)+1)
   \cdot  \E\ba{\ve{u}^2\psi_t(z_t)} + \fc{1}{4C} \sE_{p_t}\pf{q_t}{p_t}.
\end{align*}
\end{lem}
\begin{proof}
%Note that $\E\psi_t(z_t)=1$ and the normalizing factor is $\E_{p_t}\phi_t^2 = \chi^2(q_t||p_t)+1$.
By Young's inequality,
\begin{align*}
    \E\ba{\an{u ,\nb\fc{q_t(z_t)}{p_t(z_t)}}}
    &= \E\ba{\an{
    u\sfc{q_t(z_t)}{p_t(z_t)}
    ,\sfc{p_t(z_t)}{q_t(z_t)} \nb\fc{q_t(z_t)}{p_t(z_t)}}}\\
    &\le C\E\ba{\ve{u}^2 \fc{q_t(z_t)}{p_t(z_t)}
    } + \rc{4C}\E_{p_t}\ba{\ve{\nb \fc{q_t(x)}{p_t(x)}}^2
    }\\
    % & = C\E\ba{\ve{u}^2 \fc{q_t(z_t)}{p_t(z_t)}
    % } + \fc{1}{4C} \sE_{p_t}\pf{q_t}{p_t} \\
    & = C \E_{p_t}\phi_t^2\cdot  \E\ba{\ve{u}^2\psi_t(z_t)} + \fc{1}{4C} \sE_{p_t}\pf{q_t}{p_t}\\
    & = C(\chi^2(q_t||p_t)+1)\cdot  \E\ba{\ve{u}^2\psi_t(z_t)} + \fc{1}{4C} \sE_{p_t}\pf{q_t}{p_t}.
    \qedhere
\end{align*}
% \begin{align*}
%     \int q_t(x)\an{u(x), \nb \fc{q_t(x)}{p_t(x)}}dx & = \int \an{q_t(x)u(x)\cdot \fc{1}{\sqrt{p_t(x)}}, \nb \fc{q_t(x)}{p_t(x)}\cdot \sqrt{p_t(x)}}dx\\
%     & \leq c\int \fc{q_t(x)}{p_t(x)}\ve{u(x)}^2q_t(x)dx + \fc{1}{4c} \int p_t(x)\ve{\nb\fc{q_t(x)}{p_t(x)}}^2 dx\\
%     & = c\int \fc{q_t(x)}{p_t(x)}\ve{u(x)}^2q_t(x)dx + \fc{1}{4c} \sE\pf{q_t}{p_t}.
% \end{align*}
\end{proof}

\begin{lem}\label{l:chi-ineq}
%Below, let $G_t = G_{t_k,t}$ for short. 
%Define $\psi_t$ as in~\eqref{e:psi}. 
Suppose that %Assumption~\ref{a:score}
\eqref{e:score-liy}
holds for $t=T-t_k$, $\nb \ln p_{t_k}(x)$ is $L_{T-t_k}$-Lipschitz, $g$ is non-decreasing, and that $h_k\le \rc{20L_{T-t_k}g(T-t_k)^2}$. Then we have for $t\in [t_k,t_{k+1}]$ that
\begin{align*}
    \ddd t\chi^2(q_t||p_t) 
    &\le -\rc 2 
    g(T-t)^2 \FIqp 
    +   12 g(T-t)^2(\chiqpo)\cdot \\
    &\qquad 
    % \Big[
    % \sms^2 
    % \E\ba{\ve{z_t-z_{t_k}}^2\psi_t(z_t)}  
    % + \eik^2 + 
    % \E\ba{\ve{\nb \ln p_{t_k}(z_t) - \nb \ln p_t(z_t)}^2\psi_t(z_t)}\Big]
    \Bigg[
         \eik^2 + 16 G_{t_k,t}^2 L_{T-t_k}^2
    \Big[\E[\psi_t(z_t) \ve{z_t}^2]
    + 
    16 \E[\psi_t(z_t) \ve{\nb \ln p_t(x_t)}^2]\Big]\\
    &\quad 
    +64G_{t_k,t} L_{T-t_k}^2
    (8\KL(\psi_tq_t||p_t) + 2d + 16\ln 2)+
    \E\ba{\ve{\nb \ln p_{t_k}(z_t) - \nb \ln p_t(z_t)}^2\psi_t(z_t)}
    \Bigg]
    .
\end{align*}
\end{lem}
\begin{proof}
We bound the second term on the RHS of~\eqref{e:d-chi2}.
By Lemma~\ref{l:inp-young},
\begin{align}
\nonumber
    &\E\ba{\an{\score(z_{t_k}, T-t_k) -  \nb\ln p_t(z_t), \nb\fc{q_t(z_t)}{p_t(z_t)}}}\\ 
    & \le  %\cdot\E_{p_t}\phi_t^2\cdot
    (\chiqpo)
    \E\ba{\ve{\score(z_{t_k}, T-t_k) -  \nb\ln p_t(z_t)}^2\psi_t(z_t)} + \rc{4} \FIqp.
    \label{e:d-chi2-2}
\end{align}
Now
\begin{align*}
    &\ve{\score(z_{t_k}, T-t_k) -  \nb\ln p_{t}(z_t)}^2\\
    % &\le 
    % 3\ba{
    % \ve{\score(z_{t_k}, T-t_k) -   \score(z_{t}, T-t_k)}^2
    % +\ve{\score(z_t, T-t_k) - \nb\ln p_{t_k}(z_t)}^2
    % +\ve{ \nb\ln p_{t_k}(z_t) - \nb\ln p_t(z_t)}^2
    % }\\
    % &\le 
    % 3\ba{L_s^2 \ve{z_{t_k}-z_t}^2
    % +\ve{\score(z_t, T-t_k) - \nb\ln p_{t_k}(z_t)}^2
    % +\ve{ \nb\ln p_{t_k}(z_t) - \nb\ln p_t(z_t)}^2
    % }
    &\le 
    3\ba{
    \ve{\score(z_{t_k}, T-t_k) -   \nb\ln p_{t_k}(z_{t_k})}^2
    +\ve{\nb \ln p_{t_k}(z_{t_k}) - \nb\ln p_{t_k}(z_t)}^2
    +\ve{ \nb\ln p_{t_k}(z_t) - \nb\ln p_t(z_t)}^2
    }\\
    &\le 
    3\ba{\ve{\score(z_{t_k}, T-t_k) -   \nb\ln p_{t_k}(z_{t_k})}^2
    +L_{T-t_k}^2 \ve{z_{t_k}-z_t}^2
    +\ve{ \nb\ln p_{t_k}(z_t) - \nb\ln p_t(z_t)}^2
    }
\end{align*}
%\hlnote{I think we can go through $\nb \ln p_{t_k}(z_{t_k})$ instead of $s(z_t,T-t_k)$ to get dependence on $L$ rather than $L_s$.}
and
\begin{align*}
    \E\ba{\ve{\score(z_{t_k}, T-t_k)-\nb\ln p_{t_k}(z_{t_k})}^2\psi_t(z_t)} &\le \eik^2
\end{align*}
by definition of $\ep_{\iy,t}$, so by Lemma~\ref{l:diff-z},
%the result follows.
\begin{align*}
    &\E\ba{\ve{\score(z_{t_k}, T-t_k) -  \nb\ln p_t(z_t)}^2\psi_t(z_t)}\\
    &\le 
    3\ba{
        \eik^2 +
        L_{T-t_k}^2 
    \E\ba{\ve{z_t-z_{t_k}}^2\psi_t(z_t)}  
    + 
    \E\ba{\ve{\nb \ln p_{t_k}(z_t) - \nb \ln p_t(z_t)}^2\psi_t(z_t)}
    }\\
    &\le 
    3\Bigg[
        \eik^2 + 
        16 G_{t_k,t}^2 L_{T-t_k}^2
    \Big[\E[\psi_t(z_t) \ve{z_t}^2]
    + 4 \E[\psi_t(z_t)\ve{s(z_{t_k}, T-t_k) - \nb \ln p_t(z_t)}^2]\\
    &\quad + 
    16 \E[\psi_t(z_t) \ve{\nb \ln p_t(x_t)}^2]\Big]
    +64G_{t_k,t} L_{T-t_k}^2
    (8\KL(\psi_tq_t||p_t) + 2d + 16\ln 2)\\
    &\quad + 
    \E\ba{\ve{\nb \ln p_{t_k}(z_t) - \nb \ln p_t(z_t)}^2\psi_t(z_t)}
    \Bigg]
\end{align*}
The condition on $h_k$ and the fact that $g$ is non-decreasing  implies $192 G_{t_k,t}^2L_{T-t_k}^2\le \rc 2$. Rearranging gives
\begin{align*}
    &\E\ba{\ve{\score(z_{t_k}, T-t_k) -  \nb\ln p_t(z_t)}^2\psi_t(z_t)}\\
    &\le 
    6\Bigg[
        \eik^2 + 
        16 G_{t_k,t}^2 L_{T-t_k}^2
    \Big[\E[\psi_t(z_t) \ve{z_t}^2]
    + 
    16 \E[\psi_t(z_t) \ve{\nb \ln p_t(x_t)}^2\Big]\\
    &\quad 
    +64G_{t_k,t} L_{T-t_k}^2
    (8\KL(\psi_tq_t||p_t) + 2d + 16\ln 2)+ 
    \E\ba{\ve{\nb \ln p_{t_k}(z_t) - \nb \ln p_t(z_t)}^2\psi_t(z_t)}
    \Bigg]
\end{align*}
Substituting into~\eqref{e:d-chi2-2} and that inequality into~\eqref{e:d-chi2} give the conclusion.
\end{proof}

\begin{lem}\label{l:diff-z}
%Let $G_t=G_{t_k,t}$ and 
Suppose that 
$h_k\le \rc{2g(T-t_k)^2}$. %\fc{1}{3(\sms+1/2)g(T-t_k)^2}$
Then for $t\in[t_{k+1},t_k]$, 
\begin{align*}
    \E\ba{\ve{z_t-z_{t_k}}^2\psi_t(z_t)}
    &\le 
    16 G_{t_k,t}^2 
    \Big[\E[\psi_t(z_t) \ve{z_t}^2]
    + 4 \E[\psi_t(z_t)\ve{s(z_{t_k}, T-t_k) - \nb \ln p_t(z_t)}^2]\\
    &\quad + 
    16 \E[\psi_t(z_t) \ve{\nb \ln p_t(z_t)}^2]\Big]
    +64G_{t_k,t}
    (8\KL(\psi_tq_t||p_t) + 2d + 16\ln 2).
\end{align*}
\end{lem}
\begin{proof}
Consider~\eqref{e:int-interp}.
The assumption on $h_k$ implies  $G_{t_k,t}\le G_{t_k,t_{k+1}}\le \rc 2$, so $\exp\pa{\rc 2 G_{t_k,t}}-1\le G_{t_k,t}$.
Let $Y$ denote the last term of~\eqref{e:int-interp}. 
Then 
\begin{align*}
    \ve{z_t-z_{t_k}}
    &\le G_{t_k,t}\ba{\ve{z_{t_k}} + 2\ve{s(z_{t_k}, T-t_k)}} + \ve{Y}\\
    &\le 
    G_{t_k,t}\ba{\ve{z_{t}} + \ve{z_{t_k} - z_{t}} + 2\ve{s(z_{t_k}, T-t_k) - \nb \ln p_t(z_t)}  + 2\ve{\nb \ln p_t(z_t)}} + \ve{Y}.
\end{align*}
%so 
Again using $G_{t_k,t}\le \rc 2$, rearranging gives
\begin{align*}
    \ve{z_t-z_{t_k}}
    &\le  
    2G_{t_k,t}\ba{\ve{z_{t}}  + 2\ve{s(z_{t_k}, T-t_k) - \nb \ln p_t(z_t)}  + 4\ve{\nb \ln p_t(z_t)}} + 2\ve{Y},
\end{align*}
and
\begin{align*}
     \E\ba{\ve{z_t-z_{t_k}}^2\psi_t(z_t)}
    &\le 
    16 G_{t_k,t}^2 
    \Big[\E[\psi_t(z_t) \ve{z_t}^2]
    + 4 \E[\psi_t(z_t)\ve{s(z_{t_k}, T-t_k) - \nb \ln p_t(z_t)}^2]\\
    &\quad + 
    16 \E[\psi_t(z_t) \ve{\nb \ln p_t(x_t)}^2]\Big]
    +16\E[\psi_t(z_t)\ve{Y}^2].
\end{align*}
By Lemma~\ref{second_moment_diffusion}, 
\begin{align*}
    \E[\psi_t(z_t)\ve{Y}^2]
    &\le 4G_{t_k,t}(8\KL(\psi_tq_t||p_t) + 2d + 16\ln 2).
\end{align*}
The lemma follows.
\end{proof}

\begin{lem}\label{second_moment_diffusion}
For $t\in [t_k,t_{k+1}]$,
%With the setting of Lemma~\ref{d_chi2},
\begin{align*}
    &\E\ba{\psi_t(z_t)\ve{%\int_{kh}^t g(T-s)dw_s
    \int_{t_k}^t \exp\pa{\rc 2 \int_{t_k}^{t'} g(T-t'')^2\,dt''} g(t')\,  dw_{t'}
    }^2
    } 
    \\
    & \qquad \leq 2%(t-kh)g(T-kh)^2
    (\exp(G_{t_k,t})-1)
    \bigl(8\KL(\psi_tq_t||p_t) + 2d + 16\ln 2\bigr).
\end{align*}
%where $C_t$ is the LSI constant of $p_t$.
\end{lem}
\begin{proof}
Note that 
$Y:= \int_{t_k}^t \exp\pa{\rc 2 \int_{t_k}^{t'} g(T-t'')^2\,dt''} g(t')\,  dw_{t'}$ is a Gaussian random vector with variance 
\begin{align*}
    \int_{t_k}^t\exp \pa{\int_{t_k}^{t'} g(T-t'')^2\,dt''}g(t')^2\,dt' \cdot I_d
    &= \exp\pa{\int_{t_k}^{t'} g(T-t'')^2\,dt''}\Big|^{t'=t}_{t'=t_k} \cdot I_d \\
    &= (\exp(G_{t_k,t})-1)\cdot I_d.
\end{align*}
(Note that this calculation shows that the continuous-time process~\eqref{e:cts-time} does agree with the discrete-time process~\eqref{e:ei1} at $t=t_{k+1}$.)
Using the Donsker-Varadhan variational principle, for any random variable $X$,
\begin{align*}
    \Tilde{\mathbb E} X \leq \KL(\Tilde{\mathbb P}||\mathbb P) + \ln \mathbb E \exp X.
\end{align*}
Applying this to $X=c\pa{\ve{Y}-\E\ve{Y}}^2$ for a constant $c>0$ to be chosen later, and $\wt \Pj$ such that $\dd{\wt \Pj}{\Pj}(z_t) = \psi_t(z_t)$, we can bound
\begin{align}
\nonumber
    \Tilde{\mathbb E}\ve{Y}^2 & \leq
    2\E\ba{\ve{Y}^2} + 2\wt \E \ba{(Y-\E\ve{Y})^2}\\
    &\le 
    2\E\ba{\ve{Y}^2}  + \fc2c\ba{\KL(\Tilde{\mathbb P}||\mathbb P) + \ln \mathbb E \exp\pa{c\pa{\ve{Y} - \E \ve{Y}}^2}}\\
     & \leq 2d(\exp(G_{t_k,t})-1)  + \fc2c\ba{\KL(\Tilde{\mathbb P}||\mathbb P) + \ln \mathbb E \exp\pa{c\pa{\ve{Y} - \E \ve{Y}}^2}}
    .
    \label{e:dv-gaussian}
\end{align}
%where $\fc{d\Tilde{\mathbb P}}{d \mathbb P} = \psi_t(z_t)$. 
Now following \cite[Theorem 4]{chewi2021analysis}, we set $c = \fc{1}{8(\exp(G_{t_k,t})-1)}$, so that 
\begin{align*}
    \E \ba{
        \fc{(\ve{Y}-\E \ve{Y})^2}{8(\exp(G_{t_k,t})-1)}
    }
    \leq 2.
\end{align*}
Next, we have
\begin{align*}
    \KL(\Tilde{\mathbb P}||\mathbb P) & = \E_{\psi_t q_t}\ln \psi_t = \E_{\psi_t q_t}\ln \fc{\phi_t}{\E_{p_t}\phi_t^2} = \fc12\E_{\psi_t q_t}\ln \fc{\phi_t^2}{(\E_{p_t}\phi_t^2)^2}\\
    &= \fc12 \ba{\E_{\psi_t q_t}\ln \fc{\phi_t^2}{\E_{p_t}\phi_t^2} - \ln\E_{p_t}\phi_t^2}
    = \rc 2\ba{\E_{\psi_t q_t}\ln \fc{\psi_tq_t}{p_t} - \ln\E_{p_t}\phi_t^2}.
\end{align*}
Noting that $\E_{p_t}\phi_t^2 = \chi^2(q_t||p_t) + 1 \geq 1$, we have that
\begin{align*}
    \KL (\Tilde{\mathbb P}||\mathbb P) & \leq \fc12 \KL (\psi_t q_t||p_t).
    %\leq \fc{C_t}{\chi^2(q_t||p_t)+1}\cdot\sE_{p_t}\pf{q_t}{p_t},
\end{align*}
Substituting everything into~\eqref{e:dv-gaussian} gives the desired inequality.
% where the last inequality is due to Lemma~\ref{KL_psi_q&p}. We have proved
% \begin{align}
%     &\E\ba{\psi_t(z_t)\ve{\int_{kh}^t g(T-s)\,dw_s}^2}\notag \\
%     & \leq 2d\int_{kh}^t g(T-s)^2\,ds + 16\int_{kh}^t g(T-s)^2 ds\cdot\ba{\fc{C_t}{\chi^2(q_t||p_t)+1}\cdot\sE_{p_t}\pf{q_t}{p_t} + \ln 2}\notag\\
%     & \leq 2%(t-kh)g(T-kh)^2
%     \int_{kh}^t g(T-s)^2\,ds
%     \cdot\ba{\fc{8C_t}{\chi^2(q_t||p_t)+1}\cdot\sE_{p_t}\pf{q_t}{p_t} + d + 8\ln 2}. \qedhere
% \end{align}
% %where the last inequality is due to the fact that $g(t)$ is an increasing function.
\end{proof}

Let 
\begin{align}
\label{e:Kz}
    K_z &= \E\ba{\psi_t(z_t)\ve{z_t}^2}\\
\label{e:KV}
    K_V &= \E\ba{\psi_t(z_t)\ve{\nb\ln p_t(z_t)}^2}\\
\label{e:KDV}
    K_{\De V} &= \E\ba{\psi_t(z_t)\ve{\nb \ln p_{t_k}(z_t) - \nb\ln p_t(z_t)}^2}\\
\label{e:K}
    K&= \KL(\psi_tq_t||p_t).
\end{align}
In order to bound the RHS in Lemma~\ref{l:chi-ineq}, we need to bound all four of these quantities, which we do in Lemma~\ref{second_moment},~\ref{l:DV_nb_ln_p},~\ref{l:K-DeV}, and Section~\ref{s:KL}, respectively. The main innovation in our analysis compared to~\cite{lee2022convergence} is a new way to bound $K$, which we present in a separate section.
%devote Section~\ref{s:KL} to analyzing

%\subsection{Bounding $K_z$}
First we bound $K_z$. 
Recall the norm 
\[
\ve{X}_{2, \psi_2} = \inf\set{L>0}{\E e^{\fc{\ve{X}_2^2}{L^2}}\le 2}.
\]
(In other words, this is the usual Orlicz norm applied to $\ve{X}_2$.)
%\hlnote{usually given for 1-D random variables but should be ok for multi-D?}
%\jl{we could define it on $\lVert X \rVert_2$ which is the standard i think}
\begin{lem}\label{second_moment}
%With the setting of Lemma~\ref{d_chi2}, we have
For $t\in [t_k,t_{k+1}]$,
\begin{align*}
    \E\ba{\psi_t(z_t)\ve{z_t}^2} & \leq    \ve{x_t}_{2,\psi_2}^2\cdot\ba{\KL(\psi_t q_t||p_t) + \ln 2}.
\end{align*}
%where $C_t$ is the LSI constant of $p_t$, which is bounded in Lemma~\ref{l:lsi-conv}, and the second moment of $p_t$ is bounded in Lemma~\ref{2ed_moment}.
\end{lem}
\begin{proof}
By the Donsker-Varadhan variational principle,
\begin{align*}
    \E\ba{\psi_t(z_t)\ve{z_t}^2} & = \fc{2}{s}\E_{\psi_t q_t}\ba{\fc{s}{2}\ve{x}^2}\leq \fc{2}{s}\ba{\KL(\psi_t q_t||p_t) + \ln\E_{p_t}\ba{e^{\fc{s}{2}\ve{x}^2}}}
\end{align*}
for any $s>0$.
Choosing $s = 2\ve{x_t}_{2,\psi_2}^{-2}$, we have
$
    \E_{p_t}\ba{e^{\fc{s}{2}\ve{x}^2}} \leq 2, 
$
which gives the desired inequality.
% so
% \begin{align*}
%     \E\ba{\psi_t(z_t)\ve{z_t}^2} & \leq \ve{x_t}_{\psi_2}^2\cdot\ba{\KL(\psi_t q_t||p_t) + \ln 2}.
% \end{align*}
\end{proof}

The following bounds $K_V$; note that the proof does not depend on the definition of $q_t$, only that it is a probability density.
\begin{lem}[{\cite[Corollary C.7]{lee2022convergence}, \cite[Lemma 16]{chewi2021analysis}}]\label{l:DV_nb_ln_p}
%In the setting of Lemma~\ref{d_chi2}, it holds that
\begin{align*}
    \E\ba{\psi_t(z_t)\ve{\nb \ln p_t(z_t)}^2} &\leq
    \fc{4}{\chi^2(q_t||p_t)+1} \cdot\sE_{p_t}\pf{q_t}{p_t}+2dL.
\end{align*}
\end{lem}
%%%

We use the following lemma to bound $K_{\De V}$ in Lemma~\ref{second_moment_s_theta}.
\begin{lem}[{\cite[Lemma C.12]{lee2022convergence}}]\label{l:perturb_DDPM}
Suppose that $p(x)\propto e^{-V(x)}$ is a probability density on $\R^d$, where $V(x)$ is $L$-smooth. Let $p_\alpha(x) = \alpha^d p(\alpha x)$ and $\ph_{\si^2}(x)$ denote the density function of $N(0,\si^2 I_d)$. Then for $\si^2\le \rc{2\al^2L}$,  
\[
\ve{\nb \ln \fc{p(x)}{(p_\alpha *\ph_{\si^2})(x)}} \le 6\alpha^2 L\sigma d^{1/2} + (\alpha+2\alpha^3 L\sigma^2)(\alpha -1) L\ve{x} + (\alpha-1 + 2\alpha^3 L\sigma^2)\ve{\nb V(x)}.
\]
\end{lem}

\medskip 

\begin{lem}\label{second_moment_s_theta}
\label{l:K-DeV}
%With the setting of Lemma~\ref{d_chi2}, we have the following bound of the second moment of estimated score function with respect to $\psi_t q_t$:
Suppose that 
$h_k\le \rc{4Lg(T-t_k)^2}$ where $\nb \ln p_t$ is $L_{T-t}$-smooth ($L_{T-t}\ge 1$) and $L= \max_{t\in [t_k,t_{k+1}]} L_{T-t}$.
For $t\in[t_k,t_{k+1}]$, 
\begin{align*}
   & \E\ba{\psi_t(z_t)\ve{\nb\ln p_{t_k}(z_t) - \nb\ln p_{t}(z_t)}^2}\\
   & \leq 
   25L_{T-t}^2 \pa{8G_{t_k,t} d + G_{t_k,t}^2\E\ba{\psi_t(z_t)\ve{z_t}^2}}
   + 100L_{T-t}^2 \Gkht^2\E\ba{\psi_t(z_t)\ve{\nb\ln p_t(z_t)}^2}
    % 44G_t L^2 d + 44 G_t^2L^2 \E\ba{\psi_t(z_t)\ve{z_t}^2}
    % + 100 G_tL^2 \pa{\fc{4}{\chiqpo}\FIqp + 2dL}.
    %\fc{4C_{t,L}\Gkht + 8}{\chi^2(q_t||p_t)+1}\cdot\sE_{p_t}\pf{q_t}{p_t} + 4(\Mkh^2 + C_{d,L} +dL),
\end{align*}
%where $C_{t, L}$ and $C_{d,L}$ are constants defined in Lemma~\ref{l:perturb_error}.
\end{lem}
\begin{proof}
%In both SMLD and DDPM models, 
We have the following relationship for $t\in[t_k,t_{k+1}]$:
\begin{align*}
    p_{t_k} = (p_t)_\al* \ph_{\sigma^2}.
\end{align*}
where $p_\al(x)=\al^d p(\al x)$,
%. In SMLD, $\alpha=1$ and $\sigma^2 =\int_{t_k}^t g(T-s)^2\, ds$, while in DDPM, 
$\alpha = e^{\fc12\int_{t_k}^t g(T-s)^2\,ds}$ and $\sigma^2 = 1 - e^{-\int_{t_k}^t g(T-s)^2\,ds}$. %In both cases, if $h\leq \fc{1}{2g(T)^2}$, then $\sigma^2\leq \fc12$.
Observe that since $h_k\leq\fc{1}{4g(T-t_k)^2}$,
\begin{align*}
    \alpha &\leq 1+\int_{t_k}^t g(T-s)^2ds \leq 1+ h_k g(T-t_k)^2\leq 1+\rc 4\\
    \sigma^2 &= 1 - e^{-\int_{t_k}^t g(T-s)^2ds} \leq \int_{t_k}^t g(T-s)^2ds \leq h_k g(T-t_k)^2 \le  \fc14.
\end{align*}
%By Lemma~\ref{l:perturb_DDPM}, using the assumption that $L\geq 1$, we obtain
We note that 
\begin{align*}
    \si^2 \le h_k g(T-t_k)^2 \le \rc{4L_t} \le \rc{2\al^2L_t}
\end{align*}
so the hypothesis of Lemma~\ref{l:perturb_DDPM} is satisfied.
Using Lemma~\ref{l:perturb_DDPM}, we obtain
\begin{align*}
    &\E\ba{\psi_t(z_t)\ve{\nb\ln p_{t_k}(z_t) - \nb\ln p_{t}(z_t)}^2} \\
    & \leq 72\alpha^4L_{T-t}^2\sigma^2d + 4(\alpha + 2\alpha^3L_{T-t}\sigma^2)^2(\al-1)^2L_{T-t}^2\E\ba{\psi(z_t)\ve{z_t}^2} \\
    &\quad 
    + 4(\alpha - 1 + 2\alpha^3L_{T-t}\sigma^2)^2\E\ba{\psi_t(z_t)\ve{\nb\ln p_t(z_t)}^2}\\
    % &\le 
    % 72\alpha^4L^2\sigma^2d + 44L^2 \Gkht^2 \E\ba{\psi(z_t)\ve{z_t}^2}
    % + 
    % 100L^2 \Gkht^2\E\ba{\psi_t(z_t)\ve{\nb\ln p_t(z_t)}^2}
    % \\
    % & \le 44 L^2 d \Gkht
    %  + 44 L^2 \E\ba{\psi_t(z_t)\ve{z_t}^2} \Gkht^2 + 
    %100L^2 \Gkht^2\E\ba{\psi_t(z_t)\ve{\nb\ln p_t(z_t)}^2}
    &\le 72(5/4)^4L_{T-t}^2G_{t_k,t}d
    + 4(2\al)^2 G_{t_k,t}^2 L_{T-t}^2 \E\ba{\psi(z_t)\ve{z_t}^2}\\
    &\quad 
    + 4(G_{t_k,t} + 4L_{T-t}G_{t_k,t})^2 \E\ba{\psi_t(z_t)\ve{\nb\ln p_t(z_t)}^2}\\
    &\le 200L_{T-t}^2 dG_{t_k,t} + 25L_{T-t}^2 G_{t_k,t}^2 \E\ba{\psi(z_t)\ve{z_t}^2} + 100L_{T-t}^2G_{t_k,t}^2 \E\ba{\psi_t(z_t)\ve{\nb\ln p_t(z_t)}^2}.
\end{align*}
\end{proof}

%\hlnote{Putting everything together}

Now we put everything together. Write $G_t=G_{t_k,t}$ for short. Suppose $L_t$ is non-increasing. 
By Lemma~\ref{l:chi-ineq},
\begin{align*}
    \ddd t\chi^2(q_t||p_t) &\le -\rc 2 
    g(T-t)^2 \FIqp 
    +   12 g(T-t)^2(\chiqpo)\cdot E \\
    \text{where }
    E&\le 
        16G_t^2 L_{T-t_k}^2 (K_z+16K_V)+
        64G_tL_{T-t_k}^2(8K + 2d+ 16\ln 2)+
        \eik^2 + K_{\De V}
    .
\end{align*}
By Lemma~\ref{l:K-DeV}, $K_{\De V} \le 25L_{T-t}^2(8G_td+G_t^2K_z) + 100L_{T-t}^2G_t^2 K_V$, so
\begin{align*}
    E&\le 
        41G_t^2 L_{T-t}^2 K_z+
        356G_t^2 L_{T-t}^2 K_V+
        64G_tL_{T-t}^2(8K + 6d+ 16\ln 2)+
        \eik^2 
    .
\end{align*}
By Lemma~\ref{second_moment}, $K_z\le \ve{x_t}_{2,\psi_2}^2(K+\ln 2)$, 
and by Corollary~\ref{l:DV_nb_ln_p}, $K_V\le    \fc{4}{\chi^2(q_t||p_t)+1} \cdot\sE_{p_t}\pf{q_t}{p_t}+2dL$, so
\begin{align*}
    E&\le 
        41G_t^2 L_{T-t}^2 \pa{\ve{x_t}_{2,\psi_2}^2(K+\ln 2)}+
        356G_t^2 L_{T-t}^2 
        \pa{\fc{4}{\chi^2(q_t||p_t)+1} \cdot\sE_{p_t}\pf{q_t}{p_t}+2dL}\\
        &\quad 
        +
        64G_tL_{T-t}^2(8K + 6d+ 16\ln 2)+
        \eik^2 
    .
\end{align*}
Now, if $h_k\le \fc{\ep'_{h_k}}{20g(T-t_k)^2L_{T-t_{k+1}}}$, then 
\begin{align*}
    E &\le 
    {\ep'_{h_k}}^2\ba{\ve{x_t}_{2,\psi_2}^2(K+\ln 2)+
        \pa{\fc{4}{\chi^2(q_t||p_t)+1} \cdot\sE_{p_t}\pf{q_t}{p_t}+2dL_{T-t}}}\\
        &\quad 
        + 4\ep'_{h_k} L_{T-t}(8K + 2d+ 16\ln 2)
        +
        \eik^2 .
\end{align*}
Let $M_{T-t}:=\ve{x_t}_{2,\psi_2}^2$. 
Assume that $K\le \fc{A_{T-t}}{\chiqpo} + B_{T-t}$. Then we obtain
\begin{align*}
    &12 g(T-t)^2 (\chiqpo) \cdot E \\
    &\le 12 g(T-t)^2
    \Big[
        \FIqp({\ep'_{h_k}}^2  \cdot  (A_{T-t} M_{T-t} + 4) + \ep'_{h_k} \cdot 32L_{T-t}A_{T-t})
        \\
        &\hspace{8em} +(\chiqpo)({\ep'_{h_k}}^2 \cdot ((B_{T-t}+\ln 2)M_{T-t} + 2dL) \\
        &\hspace{8em} +
        \ep'_{h_k} \cdot L_{T-t} (8B_{T-t} + 6d + 16\ln 2))+\eik^2
    \Big].
\end{align*}
If $\ep'_{h_k} \le \min\bc{\rc{\sqrt{48(A_{T-t}M_{T-t}+4)}}, \fc{1}{128 L_{T-t}A_{T-t}}}$, then
\begin{align*}
    \ddd t\chi^2(q_t||p_t) 
    & \le 12 g(T-t)^2
    \Big[(\chiqpo)({\ep'_{h_k}}^2  \cdot ((B_{T-t}+\ln 2)M_{T-t} + 2dL_{T-t}) \\
    & \qquad \qquad +
        \ep'_{h_k} \cdot L_{T-t}(8B_{T-t} + 6d+16\ln 2))+\eik^2
    \Big].
\end{align*}
If $\ep'_{h_k} \le \min\bc{\fc{\sqrt{\ep'}}{g(T-t)\sqrt{24(T-t_k)((B_{T-t}+\ln 2)M_{T-t}+2dL_{T-t})}}, \fc{\ep'}{24 g(T-t)^2 (T-t) L_{T-t} (8B_{T-t} + 6d+16\ln 2)}}$, we get
\begin{align*}
    \ddd t\chi^2(q_t||p_t) 
    & \le \fc{\ep'}{T-t} (\chiqpo) + \eik^2 g(T-t)^2.
\end{align*}
Integration gives
\begin{align*}
    \chi^2(q_{t_k}||p_{t_k})
    &\le 
    e^{\ep' \int_0^{t_k} \rc{T-t}\,dt}(\chi^2(q_{0}||p_{0})+1) + \int_0^{t_k} e^{\int_t^{t_k} \fc{\ep'}{T-s}\,ds}\ep_{T-t}^2 g(T-t)^2\,dt\\
    &\le 
    \pf{T}{T-t_k}^{\ep'}\chi^2(q_0||p_0)
    + \pa{\pf{T}{T-t_k}^{\ep'} - 1}
    + %e^\ep 
    \int_0^{t_k} \pf{T-t}{T-t_k}^{\ep'} \ep_{T-t}^2 g(T-t)^2\,dt.
\end{align*}
Taking $\ep' = \fc{\ep}{\ln \pf{T}{T-t_N}}$ then gives the following Theorem~\ref{t:ddpm-liy}. We first introduce a technical assumption.
\begin{df}
Let $f:\R_{>0}\to \R_{>0}$. We say that $f$ has at most power growth and decay (with some constant $c>0$) if $\max_{u\in [\fc t2, t]}f(u) \in \ba{\fc{f(t)}{c}, cf(t)}$.
\end{df}
%\hlnote{Say something about approximating with one endpoint.}
\begin{thm}\label{t:ddpm-liy}
%There is a universal constant $c$ such that the following holds.
Suppose that the following hold.
\begin{enumerate}
    \item 
    Assumption~\ref{a:score-liy} holds for $\ep_{\iy,t}$.
    %\eqref{e:score-liy} holds for a non-increasing function $t\mapsto \ep_t$.
    \item $\ve{\wt x_t}_{2,\psi_2}^2 \le M_t$.
    \item The KL bound $
    \KL(\psi_t q_t||p_t) \le \fc{A_{T-t}}{\chi^2(q_t||p_t)+1} + B_{T-t}$ holds for any density $q_t$ and $t<t_N$, where $\psi_t(x)= \fc{q_t(x)/p_t(x)}{\chi^2(q_t||p_t)+1}$.
    \item $g(t), A_t, B_t, L_t, M_t$ have at most polynomial growth and decay (with some constant $c$).
\end{enumerate}
Then there is some constant $c'$ (depending on $c$) such that if 
% \begin{align*}
%     \ddd t\chi^2(q_t||p_t) 
%     & \le \fc{\ep}T (\chiqpo) + \ep_k^2 g(T-t)^2
% \end{align*}
% and
the step sizes satisfy 
    \begin{align*}
    h_k &\le \min\bc{\fc{T-t_k}2, \fc{c'\ep'_{h_k}}{g(T-t_k)^2 L_{T-t_k}}}, \\
    \text{where }
    \ep'_{h_k} &= 
    \min\Bigg\{
    \rc{\sqrt{A_{T-t_k}M_{T-t_k}+1}}, \fc{1}{L_{T-t_k}A_{T-t_k}},\\
    &\qquad \qquad
    \fc{\sqrt{\ep/\ln \pf{T}{T-t_N}}}{g(T-t_k)\sqrt{(T-t_k)((B_{T-t_k}+1)M_{T-t_k}+dL_{T-t_k})}}, \\
    &\qquad \qquad \fc{\ep/\ln \pf{T}{T-t_N}}{ g(T-t_k)^2 (T-t_k) L_{T-t_k} (B_{T-t_k} + d)}
    \Bigg\},
\end{align*}
then for $0\le k\le N$,
\begin{align*}
    \chi^2(q_{t_k}||p_{t_k})
    &\le e^{\ep}\chi^2(q_0||p_0)
    + (e^\ep - 1)
    + e^\ep 
    \int_0^{t_k}\ep_{\iy,T-t}^2 g(T-t)^2\,dt.
\end{align*}
%where, for $t\in [t_k,t_{k+1})$, we set $\ep_{T-t} = \ep_{T-t_k}$.
% \hlnote{Can try other settings of parameters, rather than shooting for uniform $\ep$ in time.}
\end{thm}
%\hlnote{Note that for $g(t)=\sqrt t$, $p=2$, $g(t)^3 = t^{3/2}$, $L_t^2 = \rc{\si_t^4} = O\prc{t^4}$, }
\begin{proof}
This follows from the above calculations and the observation that if we replace $F(T-t)$ by $F(T-t_k)$, for some $F$ satisfying the power growth and decay assumption, then we change the bound by at most a constant factor,  
because the step size satisfies $h_k=t_{k+1}-t_k\le \fc{T-t_k}2$.
\end{proof}

We specialize this theorem in the case of distributions with bounded support. Note that although not every initial distribution $\wt p_t$ may satisfy a KL inequality as required by condition 3 of Theorem~\ref{e:good-KL-ineq}, Lemma~\ref{l:KL-covering} will give the existence of a distribution that does, and is close in TV-error. Later in Section~\ref{s:perturb}, we show that this will have a small effect on the score function, and hence allow us to prove our main theorems.
\begin{cor}\label{c:liy-bdd}
%\hlnote{Assuming conditions of Lemma~\ref{l:KL-covering}}
Suppose that Assumptions~\ref{a:score-liy} and~\ref{a:bd} hold, $R^2\ge d$, $g\equiv 1$, and that $\wt P_0$ is such that the KL inequality~\eqref{e:good-KL-ineq} holds. Let $\de = T-t_N$. If $0<\de,\ep<\rc 2$,
%$h_k = O\pf{\ep\si_{t_{k+1}}^8}{R^4Td\ln\pa{1+\fc{\si_{t_{k+1}}}R}}$, and $0<\ep<\rc2$. 
$h_k = O\pf{\ep}{\max\{T-t_k, (T-t_k)^{-3}\}R^4d\ln \pf{T}{\de}\ln\pa{\fc R{\de\ep_K}}}$, %, and $0<\ep<\rc2$.
then %\hlnote{letting $\ep$ be the $L^\iy$ error,}
%for small enough $\ep>0$, 
for any $0\le k\le N$,
\begin{align*}
    \chi^2(q_{t_k}||p_{t_k})
    &\le e^{\ep}\chi^2(q_0||p_0)
    + \ep 
    +  e^\ep \int_0^{t_k} \ep_{\iy,T-t}^2 %g(T-t)^2
    \,dt.
\end{align*}
\end{cor}
%\hlnote{Need to deal with a conflicting $\ep$ coming from~\eqref{e:good-KL-ineq}.}
\begin{proof}
For $g\equiv 1$, note that $\si_{T-t}^2 = \Te(\min\{T-t,1\})$. From Lemma~\ref{l:Hess-bd}, we can choose 
\begin{align*}
    L_t &= \fc{R^2}{\si_t^4} = O\pf{R^2}{\min\{(T-t)^2, 1\}}. 
\end{align*}
From Lemma~\ref{l:psi2}, we can choose
\begin{align*}
    M_t &= \max\{R^2,d\}.
\end{align*}
The KL inequality~\eqref{e:good-KL-ineq} gives us 
\begin{align*}
    A_{t} &= 6(e+1)\si_t^2  = O(\min\{T-t,1\}) \\
    B_t &= \ln \prc\ep + d\ln \pa{1+O\pf{R}{\sqrt{T-t_N}}}
\end{align*}
We now check the requirements on $h_k$. We need
\begin{align}
\label{e:eh1}
    \ep'_{h_k} & = O\prc{\sqrt{A_{T-t_k}M_{T-t_k}+1}} &\Longleftarrow \ep'_{h_k} &= O\prc{\max\{R,\sqrt d\}}\\
\label{e:eh2}
    \ep'_{h_k} &= O\pf{1}{L_{T-t_k}A_{T-t_k}} &\Longleftarrow \ep'_{h_k}&= 
    O\pf{T-t_k}{R^2}\\
\label{e:eh3}
    \ep'_{h_k} &= O\pf{\sqrt{\ep/\ln \pf{T}{\de}}}{\sqrt{(T-t_k)((B_{T-t_k}+1)M_{T-t_k}+dL_{T-t_k})}}.
    \end{align}
For $T-t_k\le 1$,~\eqref{e:eh3} is implied by 
\begin{align*}
    %\Longleftarrow
    \ep'_{h_k}&= O\pf{\sqrt{\ep/\ln \pf{T}{\de}}}{\sqrt{(T-t_k)\pa{\ln\prc{\ep_K} + d\ln \pf{R}{\de}}\max\{R^2,d\} + \fc{dR^2}{T-t_k}}} \\
    \Longleftarrow \ep'_{h_k}&= O\pa{\sfc{\ep  (T-t_k)}{d\max\{R^2,d\}\ln\pf T\de \ln \pf{R}{\de\ep_K}}},
\end{align*}
and for $T-t_k>1$, 
\begin{align*}
    \ep'_{h_k}&= O\pf{\sqrt{\ep/\ln \pf{T}{\de}}}{\sqrt{T\pa{\ln \prc\ep + d\ln \pf R\de} \max\{R^2,d\} + dR^2}}\\
    \Longleftarrow
    \ep'_{h_k} &= O\pa{\sfc{\ep}{Td\max\{R^2,d\}\ln\pf T\de \ln \pf{R}{\de\ep_K}}}.
\end{align*}
Finally, the last requirement is
\begin{align*}
    \ep'_{h_k} &= O  \pf{\ep/\ln \pf{T}{\de}}{  (T-t_k) L_{T-t_k} (B_{T-t_k} + d)}\\
    \Longleftarrow \ep'_{h_k} &= O\pf{\ep}{R^2 \max\{T-t_k, (T-t_k)^{-1}\} d\ln \pf{T}{\de}\ln\pf R{\de\ep_K}}.
\end{align*}
As long as $R^2=\Om(d)$ and $\ep<1$, the last equation implies all the others. %\hlnote{FIX: the $\ep$ in the KL inequality}
    %     h_k &\le \min\bc{\fc{T-t_k}2, \fc{c'\ep'_{h_k}}{g(T-t_k)^2 L_{T-t_k}}}, \\
    % \text{where }
    % \ep'_{h_k} &= 
    % \min\Bigg\{
    % \rc{\sqrt{A_{T-t_k}M_{T-t_k}+1}}, \fc{1}{L_{T-t_k}A_{T-t_k}},\\
    % &\quad 
    % \fc{\sqrt{\ep/\ln \pf{T}{T-t_N}}}{g(T-t_k)\sqrt{(T-t_k)((B_{T-t_k}+1)M_{T-t_k}+dL_{T-t_k})}}, \\
    % &\quad \fc{\ep/\ln \pf{T}{T-t_N}}{ g(T-t_k)^2 (T-t_k) L_{T-t_k} (B_{T-t_k} + d)}
    % \Bigg\},
Plugging this into Theorem~\ref{t:ddpm-liy} gives the result.
\end{proof}
%\hlnote{Can replace $-3$ with $-p+1$ under assumption below, for $p\ge 1$.}
%\hlnote{Under $\ep_t^2=\ep_1^2/\si_t^4 = \Te(\ep_1^2/((T-t)\wedge 1)^2)$, we get $\pa{T+\rc{T-t_K}}\ep_1^2$.}

Above, we use the Hessian bound $\ve{\nb^2 \ln p_t(x)}\le \fc{R^2}{\si_t^4}$ given in Lemma~\ref{l:Hess-bd}. 
Under the stronger smoothness assumption given by Assumption~\ref{a:smoothness}, we can take the step sizes to be larger. 
% We consider the following assumption, which is Assumption A.6 in \cite{de2022convergence}. 
% \begin{asm}\label{a:smoothness}
% The following bound of the Hessian of the log-pdf holds for any $t>0$:
% \begin{align*}
%     \ve{\nb^2 \ln p_t(x)}
%     &\le \fc{C}{\si_t^2},
% \end{align*}
% for some constant $C>0$.
% \end{asm}
% Note \cite[Theorem I.8]{de2022convergence} 
% shows that if $p_0$ is a smooth density with respect to the Hausdorff measure on a convex submanifold $\mathcal M\subeq \R^d$, then $\ve{\nb\ln p_t(x)} \le \fc{C_x}{\si_t^2}$ with a constant possibly depending on $x$, though Assumption~\ref{a:smoothness} requires a uniform-in-space bound.

\begin{cor}\label{c:liy-smooth}
Suppose that Assumptions~\ref{a:score-liy},~\ref{a:bd},~\ref{a:smoothness} hold, $C\ge R^2\ge d$, $g\equiv 1$, and that $\wt P_0$ is such that the KL inequality~\eqref{e:good-KL-ineq} holds. Let $\de = T-t_N$. If $0<\de,ep<\rc 2$ and $\ep<1/\sqrt T$, 
%$h_k = O\pf{\ep\si_{t_{k+1}}^8}{R^4Td\ln\pa{1+\fc{\si_{t_{k+1}}}R}}$, and $0<\ep<\rc2$. 
$h_k = O\pf{\ep}{\max\{T-t_k, (T-t_k)^{-1}\}C^2 d\ln \pf{T}{\de}\ln\pa{\fc R{\de\ep_K}}}$, %, and $0<\ep<\rc2$.
then %\hlnote{letting $\ep$ be the $L^\iy$ error,}
%for small enough $\ep>0$, 
for any $0\le k\le N$,
\begin{align*}
    \chi^2(q_{t_k}||p_{t_k})
    &\le e^{\ep}\chi^2(q_0||p_0)
    + \ep 
    +  e^\ep \int_0^{t_k} \ep_{\iy,T-t}^2 %g(T-t)^2
    \,dt.
\end{align*}
\end{cor}
\begin{proof}
We instead have the bound $L_t = \fc{C}{\si_t^2}$. 
The requirement~\eqref{e:eh1} stays the same, while~\eqref{e:eh2} is implied by $\ep'_{h_k}=O(1/C)$. Inequality~\eqref{e:eh3}, for $T-t_k\le 1$, is implied by 
\begin{align*}
    \ep'_{h_k} &= O\pa{\sfc{1}{d\max\{C,R^2\}\ln \pf{T}{\de}\ln \pf{R}{\de \ep_K} }}.
\end{align*}
and for $T-t_k>1$, 
\begin{align*}
    \ep'_{h_k} &= O\pa{\sfc{\ep}{Td\max\{C, R^2\}\ln\pf T\de \ln \pf{R}{\de\ep_K}}}.
\end{align*}
Finally, the last requirement is implied by
\begin{align*}
    \ep'_{h_k} &= O\pa{\fc{\ep}{Cd\ln\pf T\de \ln \pf{R}{\de\ep_K}}},
\end{align*}
and for $C\ge R^2$, $\ep\le 1/\sqrt T$, implies all the others.
\end{proof}

\subsection{Auxiliary bounds}

In this section we give bounds on the Hessian ($L_t$, Lemma~\ref{l:Hess-bd}), initial $\chi^2$ divergence $\chi^2(q_0||p_0)$ (Lemma~\ref{l:K-chi}), and Orlicz norm ($M_t$, Lemma~\ref{l:psi2}).

% \begin{asm}\label{a:smoothness}
% The following bound of the Hessian of the log-pdf holds for any $t>0$:
% \begin{align*}
%     \ve{\nb^2 \ln p_t(x)}
%     &\le \fc{C}{\si_t^p},
% \end{align*}
% for some constants $C>0$ and $p\ge 0$.
% \end{asm}
% \hlnote{Only $p=2$ and $p=4$ are interesting really, I think we should only consider those 2 cases.}
% This assumption captures several interesting cases:
% \begin{itemize}
%     \item The simplest setting $p=0$ gives a uniform smoothness assumption.
%     \item The setting $p=2$ is Assumption A.6 in \cite{de2022convergence}. Note \cite[Theorem I.8]{de2022convergence} %suggests that this holds under certain convexity conditions on $p_0$: their Theorem I.8 
%     shows that if $p_0$ is a smooth density with respect to the Hausdorff measure on a convex submanifold $\mathcal M\subeq \R^d$, then $\ve{\nb\ln p_t(x)} \le \fc{C_x}{\si_t^2}$ with a constant possibly depending on $x$, though Assumption~\ref{a:smoothness} requires a uniform-in-space bound.
%     \item In the case $p_0$ is supported on a bounded set $\mathcal M\sub \R^d$, then this holds with $p=4$ and $C=1+\radius(\mathcal M)$, as shown by the following lemma.
% \end{itemize}
\begin{lem}[Hessian bound]\label{l:Hess-bd}
    Suppose that $\mu$ is a probability measure supported on a bounded set $\mathcal M\sub \R^d$ with %diameter $D$. 
    radius $R$. Then letting $\ph_{\si^2}$ denote the density of $N(0,\si^2I_d)$, 
    \begin{align}
    \label{e:Hess-bd-1}
    \ve{\nb^2 \ln (\mu * \ph_{\si^2}(x))} 
    \le \max\bc{\fc{R^2}{\si^4},\rc{\si^2}}.
    \end{align}
    Therefore, for $\wt P_0$ supported on $B_R(0)$, $R\ge 1$, we have 
    \begin{align}
    \label{e:Hess-bd-2}
    \ve{\nb^2 \ln \wt p_t(x)} 
    \le \fc{R^2}{\si_t^4}.
    \end{align}
\end{lem}
\begin{proof}
Let $\mu_{x,\si^2}$ denote the density $\mu(du)$ weighted with the gaussian $\ph_{\si^2}(u-x)$, that is, $\mu_{x,\si^2}(du) =\fc{e^{-\fc{\ve{x-u}^2}{2\si^2}} \mu(du)}{\int_{\R^d} e^{-\fc{\ve{x-u}^2}{2\si^2}} \mu(du)} $. 
We note the following calculations:
\begin{align}
\label{e:grad-ln}
    \nb \ln (\mu * \ph_{\si^2}(x)) &= 
    \fc{\nb \int_{\R^d} e^{-\fc{\ve{x-u}^2}{2\si^2}} \mu(du)}{\int_{\R^d} e^{-\fc{\ve{x-u}^2}{2\si^2}} \mu(du)}= \fc{\int_{\R^d} -\fc{x-u}{\si^2} e^{-\fc{\ve{x-u}^2}{2\si^2}} \mu(du)}{\int_{\R^d} e^{-\fc{\ve{x-u}^2}{2\si^2}} \mu(du)} = 
    -\rc{\si^2}\E_{\mu_{x,\si^2}}(x-u)\\
\label{e:grad2-ln}
    \nb^2 \ln (\mu * \ph_{\si^2}(x))
    &= \rc{\si^4}\Cov_{\mu_{x,\si^2}}(x-u) - \rc{\si^2}I_d = \rc{\si^4}\Cov_{\mu_{x,\si^2}}(x)- \rc{\si^2}I_d.
\end{align}

    The covariance of a distribution supported on a set of radius $R$ is bounded by $R^2$ in operator norm.  Inequality~\eqref{e:Hess-bd-1} then follows from~\eqref{e:grad2-ln}.
    
    For~\eqref{e:Hess-bd-2}, note that $\wt p_t = M_{m_t\sharp} \wt P_0 * \ph_{\si_t^2}$, where $m_t$ is given by~\eqref{e:OU} and $M_m$ denotes multiplication by $m$. Since $M_{m_t\sharp} \wt P_0 $ is supported on $B_{m_tR}(0)\sub B_{R}(0)$ and $\si_t\le 1$, the result follows.
\end{proof}

%\subsection{Bound on initial $\chi^2$-divergence}

\begin{lem}[Bound on initial $\chi^2$-divergence]\label{l:K-chi}
Suppose that $\wt P_0$ is supported on $B_R(0)$. 
Let $\ppr = N(0,(1-e^{G_{0,t}})I_d)$. 
Then 
\begin{align*}
    \chi^2(\ppr || \wt p_T) &\le \exp\ba{\fc{R^2\exp(-G_{0,T})}{1-\exp(-G_{0,T})}}
\end{align*}
and for $0<\ep<\rc2$ and $G_{0,T}\ge \ln \pf{4R^2}{\ep^2} \vee 1$, we have
$\chi^2(\ppr || \wt p_T)\le \ep^2$. 
\end{lem}
\begin{proof}
    We have for $x_0\sim \wt P_0$ that 
    \begin{align*}
        &\chi^2\pa{N(0,(1-e^{-G_{0,T}})I_d)|| N(x_0\exp\pa{-\rc 2 G_{0,T}}, (1-\exp(-G_{0,T}))I_d)}\\
        &\le 
        \exp\ba{\fc{\ve{x_0}^2\exp(-G_{0,T})}{1-\exp(-G_{0,T})}} 
        \le \exp\pf{R^2\exp(-G_{0,T})}{1-\exp(-G_{0,T})}
    \end{align*}
    Using convexity of $\chi^2$-divergence then gives the result. For $G_{0,T}\ge\ln \pf{4R^2}{\ep^2} \vee 1$, we have %\jl{replaced some $\ep$ to $\ep^2$ below}
    \begin{align*}
        \exp\ba{\fc{R^2\exp(-G_{0,T})}{1-\exp(-G_{0,T})}}
        &\le \exp\ba{\fc{\ep^2/4}{1/2}}\le \ep^2. \qedhere
    \end{align*}
\end{proof}

%\subsection{Other lemmas needed}

%\jl{this lemma is not cited in the proof?}
\begin{lem}[Subgaussian bound]\label{l:psi2}
Suppose $\wt P_0$ is supported on $B_R(0)$. Then for $X\sim \wt p_t$, \[\ve{X}_{2,\psi_2}\le \sqrt{\fc{e}{\ln 2}}\cdot\pa{4m_tR + 6C_1\sigma_t\sqrt{d}}
= O(\max\{R,\sqrt d\}),\]
where $m_t,\si_t$ are as in~\eqref{e:OU} and
$C_1$ is an absolute constant.
\end{lem}
\begin{proof}
Let $Y\sim \wt P_0$ s.t. $X = m_t Y + \sigma_t\xi$ for some $\xi\sim N(0, I_d)$ independent of $Y$. Define $U = \ve{X}_2 := \pa{\sum_{i=1}^d X_i^2}^{1/2}$, then for $p\geq1$,
\begin{align*}
    \E|U|^p =\E\ve{X}_2^p & \leq \E\pa{\ve{m_tY}_2 + \ve{\sigma_t \xi}_2}^p\\
    & \leq 2^{p-1}\E\ba{\ve{m_tY}_2^p + \ve{\sigma_t \xi}_2^p}\\
    & \leq 2^{p-1}\ba{(m_tR)^p + \sigma_t^p\cdot 2^{p/2}\fc{\Gamma((d+p)/2)}{\Gamma(d/2)}}\\
    & \leq 2^{p-1}\ba{(m_tR)^p + C_1 (\sqrt{2}\sigma_t)^p\cdot\pa{ d^{p/2} + p^{p/2}}}
\end{align*}
where $\Gamma$ is the commonly used gamma function and $C_1$ is an absolute constant. Therefore
\begin{align*}
    \pa{\E|U|^p}^{1/p} \leq 2m_t R + \sqrt{2}C_1\sigma_t(\sqrt{d} +\sqrt{p})\leq K\sqrt{p},
\end{align*}
where $K = 2m_tR + 3C_1\sigma_t\sqrt{d}$. Now consider $V = U/K$, then for some $\lambda>0$ small enough, by Taylor expansion,
\begin{align*}
    \E\ba{e^{\lambda^2 V^2}} = \E\ba{1 + \sum_{p=1}^\infty \fc{\pa{\lambda^2 V^2}^p}{p!}} = 1 + \sum_{p=1}^\infty \fc{\lambda^{2p}\E\ba{V^{2p}} }{p!}.
\end{align*}
Note that $\E\ba{V^{2p}}\leq (2p)^p$, while Stirling's approximation yields $p!\geq (p/e)^p$. Substituting these two bounds, we get
\begin{align*}
    \E e^{\lambda^2 V^2} \leq 1+\sum_{p=1}^\infty\pf{2\lambda^2 p}{p/e} = \sum_{p=0}^\infty (2e\lambda^2)^p =\fc{1}{1-2e\lambda^2},
\end{align*}
provided that $2e\lambda^2<1$, in which case the geometric series above converges. To bound this quantity further, we can use the numeric inequality $1/(1-x)\leq e^{2x}$ which is valid for $x\in[0,1/2]$. It follows that
\begin{align*}
    \E e^{\lambda^2 V^2} \leq e^{4e\lambda^2}\ \  \text{for all $\lambda$ satisfying $|\lambda|\leq 1/2\sqrt{e}$}.
\end{align*}
Now set $4e\lambda^2 =\ln 2$, then
\begin{align*}
    \E \ba{e^{\fc{\ln 2}{4e K^2}\ve{X}_2^2}} \leq 2,
\end{align*}
which implies that 
\begin{align*}
    \ve{X}_{2,\psi_2}\le \sqrt{\fc{4e}{\ln 2}} K &= \sqrt{\fc{e}{\ln 2}}\cdot\pa{4m_tR + 6C_1\sigma_t\sqrt{d}}. 
    \qedhere
\end{align*}
\end{proof}
%\section{Predictor}

\section{Bounding the KL divergence}
\label{s:KL}

In this section, we bound the quantity $K=\KL(\psi_tq_t||p_t)$, where $\psi_t$ is as in~\eqref{e:psi}. While $p_t$ is defined by the DDPM process, in this section we do \emph{not} assume $q_t$ is the density of the discretized process; rather, it is any density for which $\FIqp$ and $\chi^2(q_t||p_t)$ are finite.

%\fixme{[Multimodal version]}
%\hlnote{Need to make notation consistent ($\wt p$ for forward, $p$ for backward)}

% \begin{asm}\label{a:mixture}
% Suppose that $P$ is a probability measure on $\R^d$ %that is everywhere non-negative and continuously differentiable, and moreover that 
% such that 
% \begin{align}
% \label{e:p-mixture}
%     P &= \sumo jm w_jP_j,
% \end{align}
% where $w_j>0$, $\sumo jm w_j=1$, and each $P_j$ is a probability measure. %non-negative and continuously differentiable. 
% Let $w_{\min} = \min_{1\le j\le m}w_j$ and suppose all the $P_{j,t}$ satisfy a log-Sobolev constant with constant $\CLS$.
% \hlnote{Define $P_{j,t}$ as evolving for time $t$ starting from $P_j$.}
% \end{asm}

\begin{lem}\label{KL_psi_q&p}
Suppose that $\wt P_0$ is a probability measure on $\R^d$ %that is everywhere non-negative and continuously differentiable, and moreover that 
such that 
\begin{align}
\label{e:p-mixture}
    \wt P_0 &= \sumo jm w_j\wt P_{j,0},
\end{align}
where $w_j>0$, $\sumo jm w_j=1$, and each $\wt P_{j,0}$ is a probability measure. %non-negative and continuously differentiable. 
For $t>0$, let $\wt p_t$ and $\wt p_{j,t}$ be the densities obtained by running the forward DDPM process~\eqref{eq:forward_sde} for time $t$, and $p_t = \wt p_{T-t}$, $p_{j,t} = \wt p_{j,T-t}$. 
Let $w_{\min} = \min_{1\le j\le m}w_j$ and suppose all the $\wt P_{j,t}$ satisfy a log-Sobolev constant with constant $C_t$.
Then for any $q_t$, where $\psi_t$ is as in~\eqref{e:psi}
\begin{align*}
    \KL(\psi_t q_t||p_t) & \leq \fc{2C_{T-t}}{\chi^2(q_t||p_t)+1}\cdot\sE_{p_t}\pf{q_t}{p_t} + \ln \prc{w_{\min}}.
\end{align*}
\end{lem}
While we need $p_t$ to satisfy a log-Sobolev inequality to get a bound of the form $\fc{C}{\chi^2(q_t||p_t)+1}\FIqp$ (\cite[Lemma C.8]{lee2022convergence}), we note that if we allow additive slack, it suffices for $p_t$ to be a mixture of distributions satisfying a log-Sobolev inequality, with the \emph{logarithm} of the minimum mixture weight bounded below. In Lemma~\ref{l:KL-covering} we will see that we can almost decompose any distribution of bounded support in this manner, if we move a small amount of the mass.
\begin{proof}
%Let $P_{i,t}$ denote the density of the forward SDE at time $T-t$, when the initial distribution is $p_i$.
% Directly compute that
% \begin{align*}
%     \int \psi_t(x)q_t(x)dx = \int \fc{q_t(x)}{p_t(x)}q_t(x)dx/\E_{p_t}\phi_t^2 = 1.
% \end{align*}
%Clearly, $\psi_t(x) q_t(x)\geq 0$ for any $x\in\mathbb R^d$. 
%Hence $\psi_t(x)q_t(x)$ is a density function in $\mathbb R^d$. Now 
%Since $p_t$ satisfies LSI with constant $C_t$,
Let $\ol f_t:[m]\to \R$ be the function
\begin{align*}
    \ol f_t(j) &= \int_{\R^d} \fc{\psi_t(x) q_t(x)}{p_t(x)} P_{j,t}(x)\dx.
\end{align*}
By decomposition of entropy and the fact that each $P_{i,t}$ satisfies LSI with constant $C_{T-t}$, 
\begin{align*}
    \KL(\psi_t q_t||p_t) & \leq
    \Ent_{p_t} \pf{\psi_t q_t}{p_t} \\
    &=
    \sumo im \int_{\R^d} w_i \Ent_{P_{i,t}} \pf{\psi_t q_t}{p_t} 
    + \Ent_{w}(\ol f_t)\\
    & \leq \fc{C_t}{2} \sumo im w_i \sE_{P_{i,t}}\pa{\ln \fc{\psi_tq_t}{p_t}, \fc{\psi_tq_t}{p_t}} + \Ent_{w}(\ol f_t)\\ 
    & \leq \fc{C_t}{2} \sE_{p_{t}}\pa{\ln \fc{\psi_tq_t}{p_t},\fc{\psi_tq_t}{p_t}} + \Ent_{w}(\ol f_t)\\ 
    & =\fc{C_t}2 \int_{\R^d} \ve{\nb\ln\fc{\psi_t(x) q_t(x)}{p_t(x)}}^2\psi_t(x)q_t(x)\dx  + \Ent_{w}(\ol f_t)\\
    & = 2C_t\int_{\R^d} \ve{\nb \ln\fc{q_t(x)}{p_t(x)}}^2\psi_t(x)q_t(x)\dx+ \Ent_{w}(\ol f_t)\\
    & = 2C_t\int_{\R^d} \ve{\nb\fc{q_t(x)}{p_t(x)}}^2\fc{\psi_t(x)p_t(x)^2}{q_t(x)}\dx+ \Ent_{w}(\ol f_t)\\
    & = \fc{2C_t}{\chi^2(q_t||p_t)+1}\cdot \int\ve{\nb\fc{q_t(x)}{p_t(x)}}^2 p_t(x)\dx+ \Ent_{w}(\ol f_t)\\
    & \le  \fc{2C_t}{\chi^2(q_t||p_t)+1}\cdot\sE_{p_t}\pf{q_t}{p_t} + \ln \prc{w_{\min}},
\end{align*}
where the last inequality follows from noting $w_j\ol f_t(j)$ is a probability mass function on $[m]$, so that $\ol f_t(j) \le \rc{w_j}$ and 
\begin{equation*}
    \Ent_w(\ol f_t) = 
    \sumo jm w_j \ol f_t(j) \ln (\ol f_t(j))
    \le \sumo jm w_j \ol f_t(j) \ln \prc{w_{\min}} = \ln \prc{w_{\min}}. \qedhere
\end{equation*}
\end{proof}

\medskip 

\begin{lem}\label{l:KL-covering}
Suppose $0<\ep_K<\rc 2$, and  that $\ol P_0$ is a probability measure such that $\ol P_0(\cal M)\ge 1-\fc{\ep_K}{8}$. Let $\cal N\pa{\cal M, \fc{\si_t}{2}}$ denote the covering number of $\cal M$ with balls of radius $\si_t$. Given $\de>0$, there exists a distribution $\wt P_0$ such that $\chi^2(\wt P_0||\ol P_0)\le \ep_K$ and considering the DDPM process started with $\wt P_0$, for all $0\le t\le T-\de$,
\begin{align*}
    \KL(\psi_tq_t||p_t) &\le 
    \pa{\fc{6(1+e)\si_{T-t}^2}{\chiqpo}\FIqp + \ln \pf{\cal N(\cal M, \si_\de/2)}{\ep_K}}.
\end{align*}
In particular, for $\cal M=B_R(0)$ in $\R^d$, 
\begin{align}
\label{e:good-KL-ineq}
    \KL(\psi_tq_t||p_t) &\le 
    \pa{\fc{6(1+e)\si_{T-t}^2}{\chiqpo}\FIqp + \ln \prc{\ep_K} + d\ln \pa{1+\fc{4R}{\si_\de}}}.
\end{align}
\end{lem}
\begin{proof}
Partition $\cal M$ into disjoint subsets $\cal M_j$, $1\le j\le N:=\cal N(\cal M, \si_\de/2)$ of diameter at most $\si_\de$, and decompose
\begin{align*}
    \ol P_0 &= w_* P_{*} + \sumo jn \ol  w_j \wt P_{j,0} 
\end{align*}
where $p_j$ is supported on $\cal M_j$ and $ P_*=\ol P_0(\cdot|\cal M^c)$. We will zero out the coefficients of all small components: let $Z=\sum_{j:w_j\ge \fc{\ep_K}{8N}} w_j$ and 
\begin{align*}
    w_j&=\begin{cases}
        \fc{\ol w_j}{Z}, & j\in [n], \,\ol w_j\ge \fc{\ep_K}{8N}\\
        0, &\text{otherwise,}
    \end{cases}
\end{align*}
and define
\begin{align*}
    \wt P_0 &= \sumo jn w_j \wt P_{j,0} .
\end{align*}
Note that $Z \ge 1-\fc{\ep_K}8 - \sum_{j:w_j\le \fc{\ep_K}{8N}} \ge 1-\fc{\ep_K}4$. As probability distributions on $[m]\cup \{*\}$, 
\[
\chi^2(w||\ol w) \le \prc{1-\fc{\ep_K}4}^2 -1\le\ep_K,
\]
and hence the same bound holds for $\chi^2(\wt P_0||\ol P_0)$.
% By adjusting the weight vector $\ol w$ by at most $\ep_K$ in $L^1$ distance, we can obtain  
% \begin{align*}
%     \wt P_0 &= \sumo jn w_j \wt P_{j,0} 
% \end{align*}
% such that $\TV(\wt P_0,\ol P_0)\le \ep_K$ and $w_j \ge \fc{\ep_K}{n}$. 
Note each $M_{m_t \sharp} \wt P_{j,0}$ is supported on a set of diameter $m_t\si\le \si$. 
By Theorem 1 of \cite{chen2021dimension}, noting that
\begin{align*}
    \chi^2(N(\mu_2,\Si)||N(\mu_1,\Si)) &= \exp\ba{(\mu_2-\mu_1)^\top \Si^{-1} (\mu_2-\mu_1)}\le e
\end{align*}
when $\Si=\si^2I$ and $\ve{\mu_2-\mu_1}\le \si$, $\wt P_{j,t} = (M_{m_t\sharp} \wt P_{j,0}) * \ph_{\si^2}$ satisfies a log-Sobolev inequality with constant $6(1+e)\si_t^2$. The result then follows from Lemma~\ref{KL_psi_q&p}.
For $\cal M=B_R(0)$, we use the bound $\cal N(B_R(0), \si_\de/2) \le \pa{1+\fc{4R}{\si_\de}}^d$ \cite[Corollary 4.2.13]{vershynin2018high}.
\end{proof}

In the next section, we show that we can move a small amount of mass $\ep$ without significantly affecting the score function. This is necessary, as our guarantees on the score estimate are for the original distribution and not the perturbed one in Lemma~\ref{l:KL-covering}.
%\hlnote{For any distribution on a bounded set, we can move some of its mass so it satisfies the lower bound above and apply the lemma.}  

%\hlnote{Can we also depend on the radius in dissipativity? Basically, control of tails will give the multiplicative factor for $\sE_{p_t}\pf{q_t}{p_t}$, and looking at the minimum density within the radius will give the additive term. We don't really need the fact that $p_t$ is a mixture.}

\section{The effect of perturbing the data distribution on the score function}
\label{s:perturb}
In this section we consider the effect of perturbing the data distribution on the score function. The key observation is that the score function can be interpreted as the solution to an inference problem, that of recovering the original data point from a noisy sample, with data distribution as the given prior distribution. We show through a coupling argument that we can bound the difference between the score functions in terms of the distance between the two data distributions.
This will allow us to ``massage'' the data distribution in order to optimally bound $\KL(\psi_t q_t||p_t)$ in Section~\ref{s:KL}.

\subsection{Perturbation under $\chi^2$ error and truncation}

We first give a general lemma on denoising error from a mismatched prior.
%We note that the bound of Lemma~\ref{lem:mismatched_prior} can be improved under a $\chi^2$-error bound, when $K$ acts by adding noise. \hlnote{This works for unbounded distributions---no need to control tails! (only tails of gaussian, and that's easy)}
\begin{lem}[Denoising error from mismatched prior]\label{lem:mismatched_prior-chi2}
Let $\ph$ be a probability density on $\R^d$, and $P_{0,x},P_{1,x}$ be measures on $\R^d$.
For $i = 0, 1$, let $P_i$ denote the joint distribution of $x_i\sim P_{i,x}$ and $y_i = x_i+\xi_i$ where $\xi_i\sim \ph$, and let $P_{i,y}$ denote the marginal distribution of $y$.
Let  
\begin{align*}
    m^{(k)}(\ep):&=\sup_{0\le f\le 1, \int_{\R^d} f\ph \dx \le \ep}
    %\ph(A)\le \ep} 
    \int_{\R^d} f(x) \ve{x}^k \ph(x)\dx. 
\end{align*}  
Let $\etv = \TV(P_{0,x},P_{1,x})$ and $\echi^2 = \chi^2(P_{0,x}|| P_{1,x})$.
Then 
\begin{align*}
    &\int_{\R^d} P_{0,y}(dy_0)\ve{\int_{\R^d} x_0 P_0(dx_0|y_0) - \int_{\R^d} x_1 P_1(dx_1|y_0)}^2\\
    &\le  8m^{(2)}(\etv) +
    \echi\sqrt{m^{(4)}(\etv)}
    %4(m_0(2\ep)+m_0(\ep)+m_1(2\ep)+m_1(\ep)) + 2L_1^2\ew^2\\
    %&\le 
    %4(m_0(2\ep)+m_1(2\ep)) + 
    %4(1+L_1^2)(m_0(\ep)+m_1(\ep)).
\end{align*}
For $\ph=\ph_{\si^2}$, the upper bound is $O\pa{\si^2 \echi\pa{d+\ln\prc{\etv}}}$.
%We also have the upper bound $(8+2L_1^2)(m_0(2\ep)+m_1(2\ep))$.
\end{lem}
Note the tricky part of the proof is to deal with $P_1(dx_1|y_0)$, which can be thought of as inferring $x$ assuming the incorrect prior $P_{1,x}$, rather than the actual prior $P_{0,x}$.
\begin{proof}
For notational clarity, we will denote draws from the conditional distribution as $\wh x_0$ and $\wh x_1$, for example $P_0(d\wh x_0|y_0)$. 
Let $r_i(y) = \int_{\R^d} (\wh x_i-y) P_i(d\wh x_i|y)$.
%Note that the TV error between $P_{0,x}$ and $P_{1,x}$ is $\etv\le \sqrt{\chi^2(P_{0,x}||P_{1,x})}$.
Let $P_{0,1}$ be a coupling of $(x_0,y_0=x_0+\xi_0,x_1,y_1=y_1+\xi_1)$ such that $x_0=x_1$ with probability $1-\etv$ and $\xi_0=\xi_1$ with probability $1$. 
We have
\begin{align*}
    \int_{\R^d} P_{0,y}(dy_0)\ve{r_0(y_0) - r_1(y_0)}^2
    %&= 
    %\int_{\R^d\times \R^d} P_{0,1,y}(dy_0,dy_1)
    %\ve{r_0(y_0) - r_1(y_0)}^2\\
    &= 
    \ub{\int_{\{y_0=y_1\}} P_{0,1,y}(dy_0,dy_1)\ve{r_0(y_0)-r_1(y_0)}^2}{(I)}\\
    &\quad + \ub{\int_{\{y_0\ne y_1\}} P_{0,1,y}(dy_0,dy_1)\ve{r_0(y_0)-r_1(y_0)}^2}{(II)}. 
\end{align*}
Define a measure $Q$ (not necessarily a probability measure) on $\R^d$ by 
\begin{align*}
    Q(A):= P_{0,1}(y_0\in A\text{ and }y_0=y_1).
\end{align*}
Note that 
\begin{align*}
    Q(A) &\le \min\{P_{0,y}(A),P_{1,y}(A)\},
\end{align*}
so $Q$ is absolutely continuous with respect to $P_{0,y}$ and $P_{1,y}$, and by assumption on the coupling,
\begin{align}
\label{e:QRd}
    Q(\R^d) &\ge 1-\etv. 
\end{align}
Under $P_{0,1}$, when $y_0=y_1$, we can couple $P_0(d\wh x_0|y_0)$ and $P_1(d\wh x_1|y_0)$ so that $x_0=x_1$ with probability $\min\bc{\dd Q{P_{0,y}}, \dd Q{P_{1,y}}}$. %\hlnote{Explain this?} 
Let $\wh P(d\wh x_0,d\wh x_1|y_0)$ denote this coupled distribution. Then as in Lemma~\ref{lem:mismatched_prior}, 
\begin{align*}
    (I) &\le \int_{\{y_0=y_1\}} P_{0,1,y}(dy_0,dy_1) \ve{\int_{\{\wh x_0\ne \wh x_1\}}((\wh x_0 - y_0) - (\wh x_1-y_0)) \wh P(d\wh x_0, d\wh x_1|y_0)}^2\\
    &\le
    2\int_{\R^d} P_{0,1,y}(dy_0,dy_1) 
    \pa{
        \int_{\{\wh x_0\ne \wh x_1\}}\ve{\xi_0}^2 \wh P(d\wh x_0,d\wh x_1|y_0)
        +
        \int_{\{\wh x_0\ne \wh x_1\}}\ve{\xi_1}^2 \wh P(d\wh x_0,d\wh x_1|y_1)
    }
\end{align*}
We bound this by first bounding
\begin{align}
    \int_{\R^d} P_{0,1,y}(dy_1,dy_2) \wh P(\wh x_0\ne \wh x_1) &\le \int_{\R^d} P_{0,y}(dy) \max\bc{1-\dd Q{P_{0,y}}, 1-\dd Q{P_{1,y}}}\le 2\etv,
    \label{e:P-fail-couple-x}
\end{align}
which follows from the two inequalities (using~\eqref{e:QRd})
\begin{align*}
    \int_{\R^d} P_{0,y}(dy) \pa{1-\dd Q{P_{0,y}}} &= 1-Q(\R^d) \le \etv\\
    \int_{\R^d} P_{0,y}(dy) \pa{1-\dd Q{P_{1,y}}} &\le 
    \int_{\R^d} P_{1,y}(dy) \pa{1-\dd{Q}{P_{1,y}}}
    + \TV(P_{0,y},P_{1,y}) \\
    & \le (1-Q(\R^d))+\etv \le 2\etv.
\end{align*}
From~\eqref{e:P-fail-couple-x}, and the fact that the distribution of $(x_i,y_i)$ is the same as $(\wh x_i,y_i)$ by Nishimori's identity, we obtain
\begin{align*}
    (I) &\le 2(m^{(2)}(2\etv) + m^{(2)}(2\etv)) = 4m^{(2)}(\etv).
\end{align*}
%As in Lemma~\ref{lem:mismatched_prior}, 
%$\wh P()$
%We bound this by first bounding
% \begin{align}
%     \int_{\R^d} P_{0,1,y}(dy_0,dy_1) \wh P(\wh x_0\ne \wh x_1|y_0) &\le %\int_{\R^d} P_{0,y}(dy) \max\bc{1-\dd Q{P_{0,y}}, 1-\dd Q{P_{1,y}}}\le 2\ep,
%     2\ep.
%     %\label{e:P-fail-couple-x}
% \end{align}
% which follows from the two inequalities (using~\eqref{e:QRd})
% \begin{align*}
%     \int_{\R^d} P_{0,y}(dy) \pa{1-\dd Q{P_{0,y}}} &= 1-Q(\R^d) \le \ep\\
%     \int_{\R^d} P_{0,y}(dy) \pa{1-\dd Q{P_{1,y}}} &\le 
%     \int_{\R^d} P_{1,y}(dy) \pa{1-\dd{Q}{P_{1,y}}}
%     + \TV(P_{0,y},P_{1,y}) \le (1-Q(\R^d))+\ep \le 2\ep.
% \end{align*}
% From~\eqref{e:P-fail-couple-x}, and 
% From the fact that the distribution of $(x_i,y_i)$ is the same as $(\wh x_i,y_i)$ by Nishimori's identity, we obtain
% \begin{align*}
%     (A) &\le 2(m^{(2)}(2\etv) + m^{(2)}(2\etv)).
% \end{align*}
Now for the second term (II),
\begin{align*}
    (II) &\le 2\int_{\{y_0\ne y_1\}} P_{0,1,y} (dy_0,dy_1) (\ve{r_0(y_0)}^2 + \ve{r_1(y_0)}^2).
\end{align*}
The first term satisfies $\int_{\{y_0\ne y_1\}} P_{0,1,y} (dy_0,dy_1) \ve{r_0(y_0)}^2 \le m^{(2)}(\etv)$. For the second term,
we note that Cauchy-Schwarz gives for any measures $P$ and $Q$ that \begin{align*}\int_\Om f(x) P(dx)& \le 
\int_\Om f(x) Q(dx) + \int_\Om  \pa{\dd PQ-1}f(x) Q(dx)\\
&\le
\int_\Om f(x) Q(dx) + \sqrt{\chi^2(P||Q) \int_\Om f(x)^2 Q(dx)}
\end{align*} 
to switch from the measure $P_{0,y}$ to $P_{1,y}$:
\begin{align*}
    &\int_{\{y_0\ne y_1\}} P_{0,1,y} (dy_0) \ve{r_1(y_0)}^2 =
    \int_{\R^n}
    P_{0,y}(dy_0)
    P_{0,1,y} (y_0\ne y_1|y_0) \ve{r_1(y_0)}^2\\
    &\le 
    \int_{\R^n}
    P_{1,y}(dy_0)P_{0,1,y} (y_0\ne y_1|y_0)
    \ve{r_1(y_0)}^2 + 
    \sqrt{\chi^2(P_{0,y}||P_{1,y})
    \int P_{1,y}(dy_0)
    P_{0,1,y} (y_0\ne y_1|y_0)
    \ve{r_1(y_0)}^4}
\end{align*}
(Note that intentionally, the measure is $P_{1,y}$, though we use $y_0$ for the variable.)
%where the inequality for the first term follows from 
Hence,
\begin{align*}
    \int_{\R^n}
    P_{1,y}(dy_0)P_{0,1,y} (y_0\ne y_1|y_0)
    &\le \TV(P_{0,y},P_{1,y}) + \int_{\R^n}
    P_{0,y}(dy_0)P_{0,1,y} (y_0\ne y_1|y_0) \le 2\etv
\end{align*}
so
\begin{align*}
\int_{\{y_0\ne y_1\}} P_{0,1,y} (dy_0) \ve{r_1(y_0)}^2
    &\le m^{(2)}(2\etv) +
    \sqrt{\chi^2(P_{0,x}||P_{1,x})m^{(4)}(2\etv)},
\end{align*}
where we used the data processing inequality.

For $\ph=\ph_{\si^2}$, we obtain by Lemma~\ref{l:gtail} that the bound is
\begin{equation*}
    O\pa{\si^2(\etv + \echi\etv^{1/2}) \pa{d+\ln\prc{\etv}}} = 
    O\pa{\si^2 \echi \pa{d+\ln\prc{\etv}}}. \qedhere
\end{equation*}
\end{proof}

We use this lemma to obtain a bound on the $L^2$ score error under perturbation of the distribution, by interpreting the score as the solution to a de-noising problem.
%\hlnote{Note $P^0$ and $P^1$ are swapped from the above.}
\begin{lem}[$L^2$ score error under perturbation]\label{lem:l2_score_error-chi}
Let $\wt P^{(0)}=\wt P^{(0)}_0 $ and $\wt P^{(1)} = \wt P^{(1)}_0$ be two probability distributions on $\R^d$ such that $\chi^2(\wt P^{(1)}|| \wt P^{(0)})\le \echi^2\le 1$. 
\begin{enumerate}
\item 
For any $\si>0$, 
\begin{multline*}
    \int\ve{\nb \ln (\wt P^{(0)} * \ph_{\si^2})(x) - \nb \ln (\wt P^{(1)} * \ph_{\si^2})(x)}^2 (\wt P^{(1)} * \ph_{\si^2})(dx) 
    = O\pf{\echi\pa{d+\ln\prc{\echi}}}{\si^2}.
\end{multline*}
    \item 
    Let $\wt p_t^{(i)}$ be the density resulting from running~\eqref{eq:forward_sde} starting from $\wt P^{(i)}$, and let $\si_t$ be as in~\eqref{e:OU}. 
Then for any $t>0$,
\begin{align*}
    \int\ve{\nb \ln \wt p^{(0)}_t(x) - \nb \ln \wt p^{(1)}_t(x)}^2 \wt p^{(1)}_t(x)\dx 
    &= O\pf{\echi\pa{d+\ln\prc{\echi}}}{\si_t^2}.
\end{align*}
\end{enumerate}
\end{lem}
\begin{proof}
For part 1, note by~\eqref{e:grad-ln} that
\begin{align*}
    \nb \ln (\wt P^{(i)}*\ph_{\si^2})(y) &= \rc{\si^2} \E_{\wt P^{(i)}_{y,\si^2}}(x-y),
    \end{align*}
where $\wt P^{(i)}_{y,\si^2}$ is the ``tilted" probability distribution defined by
\begin{align*}
    \dd{\wt P^{(i)}_{y,\si^2}}{\wt P^{(i)}}(x) &\propto 
    e^{-\fc{\ve{x-y}^2}{2\si^2}}.
\end{align*}
By Bayes's rule, this can be viewed as the conditional probability that $x_0=x$ given $x_t=y$, where $x_0\sim \wt P^{(i)}$ and $y=x_0+\si \xi$, $\xi\sim N(0,I_d)$. Hence this fits in the framework of Lemma~\ref{lem:mismatched_prior-chi2} and %$K(x,dy) = \ph_{\si^2}(y-x)\dy$, and with prior distributions $M_{m_t\sharp} \wt P_0^{(i)}$.
\begin{align*}
&\int\ve{\nb \ln (\wt P^{(0)} * \ph_{\si^2})(y) - \nb \ln (\wt P^{(1)} * \ph_{\si^2})(y)}^2 (\wt P^{(1)} * \ph_{\si^2})(dy) \\
&=\rc{\si^4} \int_{\R^d}
\ve{\E_{\wt P^{(0)}_{y,t}} [x] -
     \E_{\wt P^{(1)}_{y,t}} [x] }^2 (\wt P^{(1)} * \ph_{\si^2})(dy)\\
     &= O\pa{\rc{\si^4} \si^2 \ep_\chi \pa{d+\ln\prc{\etv}}},
\end{align*}
giving the result.

For part 2, note that $\wt p^{(i)}_t = (M_{m_t\sharp} \wt P^{(i)}) * \ph_{\si_t^2}$. Applying part 1 with $\wt P^{(i)}\mapsfrom M_{m_t\sharp} \wt P^{(i)}$ (which preserves $\chi^2$-divergence) and $\si=\si_t$ gives the result.

% Note that for $\alpha_t:=\rc{m_t}\geq 1$ and $\sigma_t>0$, for $i=0,1$,
% \begin{align*}
%     \nb \ln \wt p^{(i)}_t(y) = 
%   \fc{\int_{\R^d} (x-y) e^{-\fc{\ve{y-x}^2}{2\si_t^2}} (M_{m_t\sharp}\wt P^{(i)})(dy)}{\sigma_t^2\int_{\R^d} \wt p^{(0)}_0(\alpha y)e^{-\fc{\ve{y-x}^2}{2\si_t^2}} (M_{m_t\sharp}\wt P^{(i)})(dy)}
%   = \fc{1}{\sigma_t^2}\E_{\wt P^{(i)}_{y, t}} \ba{x-y},
% \end{align*}
% where $\wt P^{(0)}_{y, t}$ denotes the probability distribution defined by
% \begin{align*}
% \dd{\wt P_{y,t}^{(0)}}{(M_{m_t\sharp}\wt P_0^{(0)})} (x)&\propto e^{-\fc{\ve{y-x}^2}{2\si_t^2}},
% \end{align*}
% which by Bayes's rule can be viewed as the conditional probability that $m_t x_0=x$ given $x_t=y$, where $x_0,x_t$ are generated according to the DDPM process~\eqref{eq:forward_sde} with $x_0\sim \wt P_0^{(i)}$, $y_0= m_t x_0 + \si_t \xi$, $\xi\sim N(0,I_d)$. Now
% \begin{align*}
%      \nb \ln \wt p^{(0)}_t(y) -  \nb \ln \wt p^{(1)}_t(y)
%      &= \rc{\si_t^2}\ba{\E_{\wt P^{(0)}_{y,t}} [x] -
%      \E_{\wt P^{(1)}_{y,t}} [x] },
% \end{align*}
% so this fits in the framework of Lemma~\ref{lem:mismatched_prior} with $K(x,dy) = \ph_{\si^2}(y-x)\dy$, and with prior distributions $M_{m_t\sharp} \wt P_0^{(i)}$.

% If $\wt P^{(i)}$ are supported on $B_R(0)$, then $m_i(\ep) \le \ep R^2$, which gives the second statement.
\end{proof}

Finally, we argue that a score estimate that is accurate with respect to $\wt p^{(1)}_t$ will still be accurate with respect to $\wt p^{(0)}_t$, with high probability. When using this lemma, we will substitute in the bound from Lemma~\ref{lem:l2_score_error-chi}.
\begin{lem}\label{lem:truncate_target}
Let $\wt P^{(0)}_0$ and $\wt P^{(1)}_0$ be two probability distributions on $\R^d$ with TV distance $\ep$. %satisfying the conditions of Lemma~\ref{lem:l2_score_error}. 
Suppose the estimated score function $s_t(x)$ satisfies
\[
    \ve{\nb \ln \wt p^{(0)}_t-s_t}_{L^2(\wt p^{(0)}_t)}^2=
    \E_{\wt p^{(0)}_t}\ba{\ve{\nb \ln \wt p^{(0)}_t(x) - s_t(x)}^2}\le \ep_t^2
    \]
 for $t\in(0,T]$, and 
$\nb \ln \wt p_t^{(0)}$ is $L_t$-Lipschitz. Then for $t\in(0,T]$ and any $\ep_\iy > 0$,
\[
P_{\wt p_t^{(1)}} \pa{\ve{s_t - \nb \ln \wt p_t^{(1)}}\ge \ep_\iy} \leq \ep + \fc{4}{\ep_\iy^2}\cdot\ba{\ep_t^2 + 
%\pa{4L_t^2 + 32}\cdot\pa{\delta_1\cdot R^2 + \delta_2}
\int\ve{\nb \ln \wt p^{(1)}_t(x) - \nb \ln \wt p^{(0)}_t(x)}^2 \wt p^{(1)}_t(x)\dx
}.
\]
%Recall that $\sigma_t^2 = \int_0^t g(s)^2ds$ for SMLD and $1 - e^{-\int_0^t g(s)^2\,ds}$ for DDPM.
%\hlnote{TODO: edit rest of this lemma}
\end{lem}
%\hlnote{I'm leaving the integral here to avoid repetition. In applying this can say ``by Lemma~\ref{lem:l2_score_error} and~\ref{lem:truncate_target}"}

\begin{proof}
%Fix $\ep_\iy>0$ and $t\in(0,T]$,
We have
\begin{align*}
    &P_{\wt p_t^{(1)}} \pa{\ve{s_t - \nb \ln \wt p_t^{(1)}}\ge \ep_\iy} \\& \le 
    P_{\wt p_t^{(1)}} \pa{\ve{s_t - \nb \ln \wt p_t^{(0)}}\ge \ep_\iy/2}
    + 
    P_{\wt p_t^{(1)}} \pa{\ve{\nb \ln \wt p_t^{(0)} - \nb \ln \wt p_t^{(1)}}\ge \ep_\iy/2}\\
    &\le \TV(\wt p_t^{(0)}, \wt p_t^{(1)}) + 
    P_{\wt p_t^{(0)}} \pa{\ve{s_t - \nb \ln \wt p_t^{(0)}}\ge \ep_\iy/2} + 
    P_{\wt p_t^{(1)}} \pa{\ve{\nb \ln \wt p_t^{(0)} - \nb \ln \wt p_t^{(1)}}\ge \ep_\iy/2}.
\end{align*}
%For $t\in (0,T]$,
%\begin{align*}
 %   \E_{\wt p^{(1)}_t}\ba{\ve{\nb \ln \wt p^{(1)}_t(x) - s_t(x)}^2} & = \int \ve{\nb \ln \wt p^{(1)}_t(x) - s_t(x)}^2 \wt p^{(1)}_t(x) dx\\
  %  & \leq 2 \int\ve{\nb \ln \wt p^{(1)}_t(x) - \nb \ln \wt p^{(0)}_t(x)}^2 \wt p^{(1)}_t(x)dx\\&\ \ \  + 2\int \ve{\nb \ln \wt p^{(0)}_t(x) - s_t(x)}^2 \wt p^{(1)}_t(x) dx
%\end{align*}
The first term is bounded by $\TV(\wt P^{(0)}, \wt P^{(1)})\le \ep$. 
%For the first term, using the coupling $(y_0, y_1)$ for $(\wt p^{(0)}_t, \wt p^{(1)}_t)$ in Lemma~\ref{lem:l2_score_error}, we have
%\begin{align*}
 %   \int \ve{\nb \ln \wt p^{(0)}_t(x) - s_t(x)}^2 \wt p^{(1)}_t(x) dx & = \E\ba{\ve{\nb \ln \wt p^{(0)}_t(y_1) - s_t(y_1)}^2}\\
  %  & = \E_{\{y_0=y_1\}}\ba{\ve{\nb \ln \wt p^{(0)}_t(y_0) - s_t(y_0)}^2} + \E_{\{y_0\neq y_1\}}\ba{\ve{\nb \ln \wt p^{(0)}_t(y_1) - s_t(y_1)}^2}\\
   % & \leq \E\ba{\ve{\nb \ln \wt p^{(0)}_t(y_0) - s_t(y_0)}^2} + \E_{\{y_0\neq y_1\}}\ba{\ve{\nb \ln \wt p^{(0)}_t(y_1) - s_t(y_1)}^2}\\
%    & \leq \ep_t^2 + \delta_1\cdot \fixme{M_t}.
%\end{align*}
% \begin{align*}
%     \TV(\wt p_t^{(0)}, \wt p_t^{(1)}) \leq P(y_0\neq y_1) \leq P(x_0\neq x_1) \leq \delta_1;
% \end{align*}
For the second term, by Chebyshev's Inequality,
\begin{align*}
    P_{\wt p_t^{(0)}} \pa{\ve{s_t - \nb \ln \wt p_t^{(0)}}\ge \ep_1/2} \leq \fc{4}{\ep_\iy^2} \E_{\wt p^{(0)}_t}\ba{\ve{s_t -\nb\ln \wt p^{(0)}_t}^2} \leq \fc{4\ep_t^2}{\ep_\iy^2};
\end{align*}
For the last term, again by Chebyshev's Inequality,
\begin{align*}
    P_{\wt p_t^{(1)}} \pa{\ve{\nb \ln \wt p_t^{(0)} - \nb \ln \wt p_t^{(1)}}\ge \ep_\iy/2} \leq \fc{4}{\ep_1^2}\int\ve{\nb \ln \wt p^{(1)}_t(x) - \nb \ln \wt p^{(0)}_t(x)}^2 \wt p^{(1)}_t(x)dx.
\end{align*}
%and the integral is bounded 
%by Lemma~\ref{lem:l2_score_error},
% \begin{align}
%     \int\ve{\nb \ln \wt p^{(1)}_t(x) - \nb \ln \wt p^{(0)}_t(x)}^2 \wt p^{(1)}_t(x)dx \leq \ba{\fc{4C^2}{\sigma_t^{2p}} + 32}\cdot(\delta_1\cdot R^2 +\delta_2).
% \end{align}
% And again by Chebyshev's Inequality,
% \begin{align*}
%     P_{\wt p_t^{(1)}} \pa{\ve{\nb \ln \wt p_t^{(0)} - \nb \ln \wt p_t^{(1)}}\ge \ep_1/2} \leq \fc{4}{\ep_1^2}\cdot  \ba{\fc{4C^2}{\sigma_t^{2p}} + 32}\cdot(\delta_1\cdot R^2 +\delta_2).
% \end{align*}
We conclude the proof by combining the these three inequalities.
%\hlnote{Use the interpretation of $\nb \ln p_t(y)$ as the MMSE estimator for $x$ given that $y$ is generated from evolving $x$ for time $t$.}
\end{proof}
%%%%%%%%%%%%%%

%\hlnote{Need this to get TV bound.}
Finally, we will need the following to obtain a TV error bound to $\wt p_0$ in Theorem~\ref{t:tv-tail}.
\begin{lem}\label{l:bdd_spt}
Suppose that $\wt p_0\propto e^{-V(x)}$ is a probability density on $\R^d$ with bounded first moment $\E_{\wt p_0}\ve{X}$, and $V$ is $L$-smooth.
Then for $t>0$ such that $\alpha_t\sigma_t\leq \rc{2L}$, we have
\begin{align*}
    \TV(\wt p_t, \wt p_0)\le 2\pa{\alpha_t - 1}\cdot\pa{L\E_{\wt p_0}\ve{X} + d} +  \fc32 dL\alpha_t\sigma_t.
\end{align*}
Here $\alpha_t = 1/m_t$ and $\sigma_t$ are defined in~\eqref{e:OU}. %If we further assume that $\wt p_0$ is supported on $B_R(0)$ and $V$ is $L$-smooth on the interior of $B_R(0)$, then $\TV(\wt p_t, \wt p_0)\le \pa{LR+d}\pa{\alpha_t - 1} + \fc32 dL\alpha_t\sigma_t$. 
In particular, $TV(\wt p_\delta, \wt p_0)\leq \ep_{\TV}$ if $\delta = O(\fc{\ep_{TV}^2}{R^2L^2})$ and $R=\max\bc{\sqrt{d}, \E_{\wt p_0}\ve{X}}$.
\end{lem}
\begin{proof}
Without loss of generality, we assume that $\wt p_0(x) = e^{-V(x)}$. Note that $\wt p_t(x) = \int \alpha_t^d\wt p_0(\alpha_t y) \ph_{\sigma_t^2}(x-y)\,dy$. Let $\wt q_t(x) := \alpha_t^d\wt p_0(\alpha_t x)$, which is also a probability density on $\R^d$. Then by the triangle inequality,
\begin{align*}
    \TV (\wt p_t, \wt p_0) \leq \TV (\wt p_t, \wt q_t) + \TV (\wt q_t, \wt p_0).
\end{align*}
For the second term,
\begin{align*}
    \ab{\wt q_t(x) - \wt p_0(x)} & = \ab{\alpha_t^d\wt p_0(\alpha_t x) - \wt p_0(x)}\\
    & =  \ab{e^{-V(\alpha_t x) + d\ln \alpha_t} - e^{-V(x)}}\\
    & \leq \max\bc{e^{-V(x)}, e^{-V(\alpha_t x) + d\ln \alpha_t}}\cdot \pa{1- e^{-\ab{V(x) - V(\alpha_t x) + d\ln\alpha_t}}}\\
    & \leq  \pa{\wt p_0(x) + \wt q_t(x)}\cdot\pa{\ab{V(x) - V(\alpha_t x)} + d\ln \alpha_t}\\
    & \leq \pa{\wt p_0(x) + \wt q_t(x)}\cdot\ba{L\ve{x}\pa{\alpha_t -1} + d\ln\alpha_t},
\end{align*}
%And by Chebyshev's inequality, for any $X\sim\wt p_0$
%\begin{align*}
 %   \int_{B_R(0)^c} \ab{\wt q_t(x) - \wt p_0(x)}dx \leq  \int_{B_R(0)^c}\wt q_t(x)dx + \int_{B_R(0)^c}\wt p_0(x)dx \leq \fc{\E_{\wt p_0}{\ve{m_t X}}}{R} + \fc{\E_{\wt p_0}{\ve{ X}}}{R} \leq \fc{2\E_{\wt p_0}\ve{X}}{R}.
%\end{align*}
where in the second inequality, we use the fact that $1 - e^x\leq \ab{x}$ for all $x\leq 0$. Thus
\begin{align*}
    \TV(\wt q_t(x) , \wt p_0(x)) & = \fc12\int |\wt q_t(x) - \wt p_0(x)|\dx\\
    &\leq \int \ba{L\pa{\alpha_t -1}\ve{x} + d\ln\alpha_t} \wt p_0(x)\dx +\int \ba{L\pa{\alpha_t -1}\ve{x} + d\ln\alpha_t} \wt q_t(x)\dx\\
    & \leq L(\alpha_t -1)\pa{\int \ve{x}\wt p_0(x)dx + \int \ve{x}\wt q_t(x)dx} + 2d\ln\alpha_t\\
    & \leq 2L(\alpha_t - 1)\int \ve{x}\wt p_0(x)dx + 2d\ln\alpha_t.
\end{align*}
%In particular, if $\wt p_0$ is supported on $B_R(0)$, then 
%\begin{align*}
    %\TV(\wt q_t(x) , \wt p_0(x)) \leq LR(\alpha_t -1) + d\ln\alpha_t.
%\end{align*}
Now for the first term,
\begin{align*}
    \wt p_t(x) - \wt q_t(x) & = \int \wt q_t(x-y)\ph_{\sigma_t^2}(y)\dy - \wt q_t(x) = \int \pa{\wt q_t(x - \sigma_t y) - \wt q_t(x)}\ph(y)dy,
\end{align*}
where $\ph(y)$ is the density of the $d$-dimensional standard Gaussian distribution. Apply Minkowski's inequality for integrals: %\hlnote{Again have to be careful with cutoff}
\begin{align*}
    \int \ab{\wt p_t(x) - \wt q_t(x)}dx & = \int \ab{\int \pa{\wt q_t(x - \sigma_t y) - \wt q_t(x)}\ph(y)dy}dx\\
    & \leq \int \ba{\int \ab{\wt q_t(x - \sigma_t y) - \wt q_t(x)} dx}\ph(y)\dy \\
    & \leq \int  \ba{\int \pa{e^{L\ve{\alpha_t\sigma_t y}}-1}\wt q_t(x) dx}\ph(y)\dy\\
    & =  \int \pa{e^{L\ve{\alpha_t\sigma_t y}}-1}\ph(y)\dy\\
    & = \pa{2\pi}^{-d/2}\int e^{L\alpha_t\sigma_t\ve{y} - \fc{\ve{y}^2}{2}}\dy - 1\\
    & \leq \pa{2\pi}^{-d/2}\int e^{
    \ba{-\fc12 + \pa{L\alpha_t\sigma_t}^2}\ve{y}^2}dy  + L\alpha_t\sigma_t\int\ve{y}\ph(y)\dy -1\\
    & \leq \ba{\fc{1}{1-2\pa{L\alpha_t\sigma_t}^2}}^{d/2} + \sqrt dL\alpha_t\sigma_t - 1\\
    & \leq e^{2d\pa{L\alpha_t\sigma_t}^2} - 1 + \sqrt dL\alpha_t\sigma_t\\
    & \leq 4d\pa{L\alpha_t\sigma_t}^2 + \sqrt dL\alpha_t\sigma_t,
\end{align*}
where in the third inequality, we use the elementary inequality $e^x\leq x + e^{x^2}$, which is valid for all $x\in\R$, and in the fifth inequality, we use $\fc{1}{1-2x} \leq e^{4x}$, which holds for $x\in[0,1/3]$. Hence if $L\alpha_t\sigma_t \leq 1/2$, we have
\begin{align*}
    \TV(\wt p_t, \wt q_t)\leq \fc12 \int \ab{\wt p_t(x) - \wt q_t(x)}dx \leq \fc32 dL\alpha_t\sigma_t.
\end{align*}
%On the other hand, if $\wt p_0$ is supported in $B_{R}(0)$, then
%\begin{align*}
 %   \ab{\wt q_t(x-\sigma_t y) - \wt q_t(x)} & \leq  \begin{cases}
  %  \wt q_t(x)\pa{e^{L\alpha_t\sigma_t\ve{y}} - 1}, & x, x-\sigma_t y\in B_{m_t R}(0)\\
  %  0, & x, x-\sigma_t y\nin B_{m_t R}(0)\\
  %  \wt q_t(x-\sigma_t y) + \wt q_t(x), & \text{otherwise}.
  %  \end{cases}
%\end{align*}
%Denote the 'otherwise' case by $A\subset \R^d\times \R^d$, then by symmetry,
%\begin{align*}
 %   \int_A \ba{\wt q_t(x-\sigma_t y) + \wt q_t(x)} \phi(y)dxdy & = 2\int_{\ve{x}< m_t R}\ba{\int_{\ve{x-\sigma_t y}>m_t R} \ba{\wt q_t(x-\sigma_t y) + \wt q_t(x)} \phi(y) dy}dx\\
 %   & = 2\int_{\ve{x}< m_t R}\ba{\int_{\ve{x-\sigma_t y}>m_t R} \wt q_t(x) \phi(y) dy}dx\\
 %   & = 2\int_{\ve{x}< m_t R}\ba{\int_{\ve{x-\sigma_t y}>m_t R}\phi(y) dy}\wt q_t(x) dx
%\end{align*}
Now we conclude the proof by combining the bounds for $\TV(\wt p_t, \wt q_t)$ and $\TV(\wt p_0, \wt q_t)$:
\begin{align*}
    \TV (\wt p_t, \wt p_0) & \leq \TV (\wt p_t, \wt q_t) + \TV (\wt q_t, \wt p_0) \\&\leq 2L(\alpha_t - 1)\int \ve{x}\wt p_0(x)dx + 2d\ln\alpha_t + \fc32 dL\alpha_t\sigma_t\\&\leq 2\pa{\alpha_t - 1}\cdot\pa{L\E_{\wt p_0}\ve{X} + d} +  \fc32 dL\alpha_t\sigma_t,
\end{align*}
where we use the fact that $\ln x \leq x -1$ for all $x\geq 1$. Recall that $\alpha_t = 1/m_t = e^{t/2}$ and $\sigma_t^2 = 1-e^{-t}$ when $g\equiv 1$. %Now for $\wt p_0$ supported in $B_R(0)$, let
It suffices for
\begin{align*}
    \max \bc{2\pa{L\E_{\wt p_0}\ve{X} + d}\pa{\alpha_\delta -1}, \fc32 dL\alpha_\delta\sigma_\delta}\leq \fc{\ep_{\TV}}{2},
\end{align*}
which is implied by
\begin{align*}
    \delta &\precsim  {\min\bc{\fc{\ep_{\TV}}{L\E_{\wt p_0}\ve{X}+d}, \fc{\ep_{\TV}^2}{d^2L^2}}} \asymp {\fc{\ep_{\TV}^2}{R^2L^2}},
\end{align*}
for appropriate constants,
as $R\geq \max\bc{\sqrt{d}, \E_{\wt p_0}\ve{X}}$.
\end{proof}

\subsection{Perturbation under TV error}
%%%%%%%%%%%%%%
%We first give a general lemma on denoising error from a mismatched prior.
Although we will not need it in our proof, we note that we can derive a similar perturbation result under TV error, which might be of independent interest. 
\begin{lem}\label{lem:mismatched_prior}
Let $K(x,dy)$ be a probability kernel on $\R^d$, let $P_{0,x},P_{1,x}$ be measures on $\R^d$.
Let $P_i$ denote the joint distribution of $x_i\sim P_{i,x}$ and $y_i\sim K(x_i,\cdot)$, and let $P_{i,y}$ denote the marginal distribution of $y$.
Suppose there is a coupling $P_{0,1}$ of $(x_0,y_0)\sim P_0$ and $(x_1,y_1)\sim P_1$ such that 
\begin{itemize}
    \item $x_0=x_1$ with probability $1-\ep$,
    \item $x_0=x_1$ implies $y_0=y_1$, and
    \item $\E[\ve{y_0-y_1}^2]\le \ew^2$.
\end{itemize}
Define the tail error by 
%$m_i(\ep)$, where $m_i$ is the concave hull\footnote{This extra step is unnecessary if $P_{i,x}$ are non-atomic measures. If $P_{i,x}$ have atoms, this is necessary as the set constructed in the coupling may contain part of an atom of $P$. In general, we can find a probability space $P_{i,x}^{\text{nonatomic}}$ that is non-atomic and a random variable $X_i$ on  this space with distribution $P_{i,x}$. It's easy to check that 
%$m_i(\ep) = \sup_{A:P_{i,x}^{\text{nonatomic}}(A)\le \ep} \int_A \ve{X}^2 dP_{i,x}(\om)$.
%} of 
% \begin{align*}
%     \ol m_i(\ep):&=\sup_{A:P_{i,x}(A)\le \ep} \int_A \ve{x}^2 P_{i,x}(dx). 
% \end{align*}
\begin{align*}
    m_i(\ep):&=\sup_{0\le f\le 1, \int_{\R^d} f\ph \dx \le \ep}
    %\ph(A)\le \ep} 
    \int_{\R^d} f(x) \ve{x}^2 P_i(\dx). 
%    m_i(\ep):&=\sup_{X\sim P_{i,x} , \mu(A)\le \ep} \E_\mu[\ve{X}^2\one_A],
\end{align*}
%where the sup is over random variables $X$ and $X$ is defined on a probability space with measure $\mu$.\footnote{Note we define it this way because in the case that $P_{i,x}$ has atoms, $A$ may not be $X$-measurable.} 
Let $r_i(y) = \int_{\R^d} x_i P_i(dx_i|y)$, and suppose that $r_1(y) = \int_{\R^d} x_1 P_1(dx_1|y)$ is $L_1$-Lipschitz.  
Then 
\begin{align*}
    &\int_{\R^d} P_{0,y}(dy_0)\ve{\int_{\R^d} x_0 P_0(dx_0|y_0) - \int_{\R^d} x_1 P_1(dx_1|y_0)}^2\\
    & \qquad \le 4(m_0(2\ep)+m_0(\ep)+m_1(2\ep)+m_1(\ep)) + 2L_1^2\ew^2\\
    & \qquad \le 
    4(m_0(2\ep)+m_1(2\ep)) + 
    4(1+L_1^2)(m_0(\ep)+m_1(\ep)).
\end{align*}
%We also have the upper bound $(8+2L_1^2)(m_0(2\ep)+m_1(2\ep))$.
\end{lem}
% Note the tricky part of the proof is to deal with $P_1(dx_1|y_0)$, which can be thought of as inferring $x$ assuming the incorrect prior $P_{1,x}$, rather than the actual prior $P_{0,x}$.
\begin{proof}
For notational clarity, we will denote draws from the conditional distribution as $\wh x_0$ and $\wh x_1$, for example $P_0(d\wh x_0|y_0)$.
%%  Let $P_{0,1}$ be a coupling of $(x_0,y_0)\sim P_0$ and $(x_1,y_1)\sim P_1$ that has the specified behavior as a coupling of $P_{0,x}$ and $P_{1,x}$, and such that if $x_0=x_1$, then $y_0=y_1$. \yt{Does such a coupling always exist? Any reference? } \hlnote{It's because the $y_i$'s are generated from the $x_i$'s with the same kernel.}
We have
\begin{align*}
    %  &\int_{\R^d} P_{0,y}(dy)\ve{\int_{\R^d} x_0 P_0(dx_0|y_0) - \int_{\R^d} x_1 P_1(dx_1|y_0)}^2 \\
    %  &\le 
    %  \int_{\R^d} P_{0,1,y}(dy_0,dy_1)\ve{\int_{\R^d} x_0 P_0(dx_0|y_0) - \int_{\R^d} x_1 P_1(dx_1|y_1)}^2\\
    %  &\quad 
    % + \int_{\R^d} P_{0,1,y}(dy_0,dy_1)\ve{\int_{\R^d} x_1 P_1(dx_1|y_1) - \int_{\R^d} x_1 P_1(dx_1|y_0)}^2. 
    \int_{\R^d} P_{0,y}(dy_0)\ve{r_0(y_0) - r_1(y_0)}^2 &\le 
     2\ub{\int_{\R^d\times \R^d} P_{0,1,y}(dy_0,dy_1)\ve{r_0(y_0)-r_1(y_1)}^2}{(I)}\\
    &\quad + 2\ub{\int_{\R^d\times \R^d} P_{0,1,y}(dy_0,dy_1)\ve{r_1(y_1)-r_1(y_0)}^2}{(II)}. 
\end{align*}
%Let $B_{y_0,y_1} = \set{(y_0,y_1)}{y_0\ne y_1}$. 
For the first term (I), we split it as
\begin{align*}
    (I) &\le \ub{\int_{%B_{y_0,y_1}^c
    \{y_0=y_1\}} P_{0,1,y}(dy_0,dy_1) \ve{r_0(y_0)-r_1(y_0)}^2}{(i)}
    + \ub{\int_{%B_{y_0,y_1}
    \{y_0\ne y_1\}} P_{0,1,y}(dy_0,dy_1) \ve{r_0(y_0)-r_1(y_1)}^2}{(ii)}.
\end{align*}
Define the measure $Q$ on $\R^d$ by 
\begin{align*}
    Q(A):&= P_{0,1}(y_0\in A\text{ and }y_0=y_1).
\end{align*}
% Note that 
% \begin{align*}
%     Q(A) &\le \min\{P_{0,y}(A),P_{1,y}(A)\},
% \end{align*}
% so $Q$ is absolutely continuous with respect to $P_{0,y}$ and $P_{1,y}$, and by assumption on the coupling,
% \begin{align}
% \label{e:QRd}
%     Q(\R^d) &\ge 1-\ep. 
% \end{align}
As in Lemma~\ref{lem:l2_score_error-chi}, %%
under $P_{0,1}$, when $y_0=y_1$, we can couple $P_0(d\wh x_0|y_0)$ and $P_1(d\wh x_1|y_0)$ so that $x_0=x_1$ with probability $\min\bc{\dd Q{P_{0,y}}, \dd Q{P_{1,y}}}$. %\hlnote{Explain this?} 
Let $\wh P(d\wh x_0,d\wh x_1|y_0)$ denote this coupled distribution. Then 
\begin{align*}
    (i) &\le \int_{\{y_0=y_1\}} P_{0,1,y}(dy_0,dy_1) \ve{\int_{\{\wh x_0\ne \wh x_1\}}(\wh x_0 - \wh x_1) \wh P(d\wh x_0, d\wh x_1|y_0)}^2\\
    &\le
    2\int_{\R^d} P_{0,1,y}(dy_0,dy_1) 
    \pa{
        \int_{\{\wh x_0\ne \wh x_1\}}\ve{\wh x_0}^2 \wh P(d\wh x_0,d\wh x_1|y_0)
        +
        \int_{\{\wh x_0\ne \wh x_1\}}\ve{\wh x_1}^2 \wh P(d\wh x_0,d\wh x_1|y_1)
    }\\
    &\le 2(m_0(2\ep)+m_1(2\ep))
\end{align*}
as in Lemma~\ref{lem:l2_score_error-chi}.
% We bound this by first bounding
% \begin{align}
%     \int_{\R^d} P_{0,1,y}(dy_1,dy_2) \wh P(\wh x_0\ne \wh x_1) &\le \int_{\R^d} P_{0,y}(dy) \max\bc{1-\dd Q{P_{0,y}}, 1-\dd Q{P_{1,y}}}\le 2\ep,
%     \label{e:P-fail-couple-x}
% \end{align}
% which follows from the two inequalities (using~\eqref{e:QRd})
% \begin{align*}
%     \int_{\R^d} P_{0,y}(dy) \pa{1-\dd Q{P_{0,y}}} &= 1-Q(\R^d) \le \ep\\
%     \int_{\R^d} P_{0,y}(dy) \pa{1-\dd Q{P_{1,y}}} &\le 
%     \int_{\R^d} P_{1,y}(dy) \pa{1-\dd{Q}{P_{1,y}}}
%     + \TV(P_{0,y},P_{1,y}) \le (1-Q(\R^d))+\ep \le 2\ep.
% \end{align*}
% From~\eqref{e:P-fail-couple-x}, and the fact that the distribution of $(x_i,y_i)$ is the same as $(\wh x_i,y_i)$ by Nishimori's identity, we obtain
% \begin{align*}
%     (i) &\le 2(m_0(2\ep) + m_1(2\ep)).
% \end{align*}
Now
\begin{align*}
    (ii) &\le 2\int_{\{y_0\ne y_1\}} P_{0,1,y} (dy_0,dy_1) (\ve{r_0(y_0)}^2 + \ve{r_1(y_1)}^2)
    \le 2(m_0(\ep)+m_1(\ep)).
\end{align*}

% Considering the distribution $P(\cdot |y_0,y_1)$, let $B_{y_0,y_1}$ be the set where $x_0\ne x_1$. By assumption, $P_{i,x}(B_{y_0,y_1}) \leq \ep$, $i=0,1$. For the first term, note that
% \begin{align*}
%     r_0(y_0)-r_1(y_1) & = \int_{B_{y_0,y_1}} \ba{x_0P_0(dx_0|y_0) - x_1P_1(dx_1|y_1)} + \int_{B_{y_0,y_1}^c}  \ba{x_0P_0(dx_0|y_0) - x_1P_1(dx_1|y_1)}\\
%     & = \int_{B_{y_0,y_1}} \ba{x_0P_0(dx_0|y_0) - x_1P_1(dx_1|y_1)},
% \end{align*}
% where the second equality is based on the observation that $x_0=x_1$ on $B_{y_0,y_1}^c$. Hence we apply the coupling to get
% \begin{align*}
%     %&\int_{\R^d} P_{0,1,y}(dy_0,dy_1)\ve{\int_{\R^d} x_0 P_0(dx_0|y_0) - \int_{\R^d} x_1 P_1(dx_1|y_1)}^2\\
%     \int_{\R^d} P_{0,1,y}(dy_0,dy_1)\ve{r_0(y_0)-r_1(y_1)}^2
%     & = 
%     \int_{\R^d} P_{0,1,y}(dy_0,dy_1)
%     \ve{\int_{B_{y_0,y_1}} [x_0P_0(dx_0|y_0) - x_1P_1(dx_1|y_1)]
%     }^2\\
%     & \leq 2\int_{\R^d} P_{0,1,y}(dy_0,dy_1)
%     \ve{\int_{B_{y_0,y_1}} x_0P_0(dx_0|y_0)
%     }^2 \\
%     &\ \ \ + 2\int_{\R^d} P_{0,1,y}(dy_0,dy_1)
%     \ve{\int_{B_{y_0,y_1}} x_1P_1(dx_1|y_1)
%     }^2\\
%     &\leq 2m_0(\ep) + 2m_1(\ep).
% \end{align*}
% \hlnote{split into when $y_0=y_1$ and when not. need the wasserstein condition on $y_1$, this is a bit complicated.}
% \yt{What do you mean by wasserstein condition?}
Finally, for the second term (II), we use the fact that $r_1$ is $L_1$ Lipschitz and the coupling to conclude
\begin{align*}
    % \int_{\R^d} P_{0,1,y}(dy_0,dy_1)\ve{\int_{\R^d} x_1 P_1(dx_1|y_1) - \int_{\R^d} x_1 P_1(dx_1|y_0)}^2
    %\int_{\R^d} P_{0,1,y}(dy_0,dy_1)\ve{r_1(y_1)-r_1(y_0)}^2
    (II)
    &\le 
    \int_{\R^d} P_{0,1,y}(dy_0,dy_1)L_1^2 \ve{y_0-y_1}^2 \le L_1^2\ew^2.
\end{align*}
% \yt{Does the last inequality hold for a general probability kernel $K$? It seems that we need some contraction property.} \hlnote{We assume Lipschitzness of $r_1$ (which is Lipschitzness of the score function in our application).}
% \yt{But then how do we bound $\ve{y_0 -y_1}$ in terms of $\ew$?}
% \hlnote{Assumption should be on $y$'s, not $x$'s. (It's the same in our application because we can couple the Gaussian noise.)} \yt{I see.}
We conclude the proof by combining the inequalities for (i), (ii), and (II). %two inequalities above.

For the second upper bound, we note that 
\begin{equation*}
\E[\ve{y_0-y_1}^2]\le 2(\E[\ve{y_0}^2] + \E[\ve{y_1}^2]) \le 2(m_0(\ep)+m_1(\ep)). \qedhere
\end{equation*}
\end{proof}

\subsection{Gaussian tail calculation}
We use the following Gaussian tail calculation in the proof of Lemma~\ref{lem:l2_score_error-chi}.
\begin{lem}
\label{l:gtail}
Let $\mu$ be the standard Gaussian measure on $N(0,I_d)$. Then
\begin{align*}
    \sup_{\mu(A)\le \ep} \int_A\ve{x}^2\mu(dx) &\le \ep\pa{2d+3\ln\prc\ep+3}
    =O\pa{\ep\pa{d+\ln\prc\ep}}
    \\
    \sup_{\mu(A)\le \ep} \int_A\ve{x}^4\mu(dx) &\le \ep\pa{2d+3\ln\prc\ep}^2  + 3\ep\pa{2d+3\ln\prc\ep} + 9\ep=
    O\pa{\ep\pa{d^2+\ln\prc\ep^2}}.
\end{align*}
\end{lem}
\begin{proof}
By the $\chi^2$ tail bound in~\cite{laurent2000adaptive}, for $t\ge 0$,
\begin{align}
    \mu(\ve{X}^2 \ge 2d + 3t)
    &\le \Pj(\ve{X}^2 \ge d+2\sqrt{dt}+2t) \le e^{-t},
\end{align}
so $\ve{X}^2$ is stochastically dominated by a random variable with cdf $F(y) = 1-e^{-\fc{y-2d}3}$. Then letting $P_Y$ be the measure corresponding to $F$, 
\begin{align*}
    \sup_{\mu(A)\le \ep} \int_A\ve{x}^2\mu(dx) &\le 
    \sup_{P_Y(A)\le \ep} 
    \int_A y P_Y(dy)=
    \int_{2d+3\ln\prc\ep}^{\iy} y dF(y)\\
    &=
    \ep\pa{2d+3\ln\prc\ep} + 
    \int_{2d+3\ln\prc\ep}^{\iy} e^{-\fc{y-2d}3}dy 
    = \ep\pa{2d+3\ln\prc\ep} + 3\ep
\end{align*}
and
\begin{align*}
    \sup_{\mu(A)\le \ep} \int_A\ve{x}^4\mu(dx) &\le 
    \sup_{P_Y(A)\le \ep} 
    \int_A y^2 P_Y(dy)=
    \int_{2d+3\ln\prc\ep}^{\iy} y^2 dF(y)\\
    &= 
    \ep\pa{2d+3\ln\prc\ep}^2 + \int_{2d+3\ln\prc\ep}^{\iy} 2ye^{-\fc{y-2d}3}dy\\
    &= 
    \ep\pa{2d+3\ln\prc\ep}^2 
    -\left[3ye^{-\fc{y-2d}{3}}\right]\Big|^\iy_{2d+3\ln\prc\ep}
    + \int_{2d+3\ln\prc\ep}^\iy 3 e^{-\fc{y-2d}3}\,dy \\
    &= 
    \ep\pa{2d+3\ln\prc\ep}^2  + 3\ep\pa{2d+3\ln\prc\ep} + 9\ep. \qedhere
\end{align*}
\end{proof}

\section{Guarantees under $L^2$-accurate score estimate}
\label{s:l2}

We will state our results under a more general tail bound assumption.
\begin{asm}[Tail bound]
\label{a:tail}
$R:[0,1]\to [0,\iy)$ is a function such that $\pdata(B_{R(\ep)}(0))\ge 1-\ep$.
\end{asm}
Our result will require $R(\ep)$ to grow at most as a sufficiently small power of $\ep^{-1}$ as $\ep\to 0$; in particular, this holds for subexponential distributions.
By taking $R$ to be a constant function, this contains the assumption of bounded support (Assumption~\ref{a:bd}) as a special case.

\subsection{TV error guarantees}

We follow the framework of~\cite{lee2022convergence} to convert guarantees under $L^\iy$-accurate score estimate, to guarantees under $L^2$-accurate score estimate. 

\begin{thm}[{\cite[Theorem 4.1]{lee2022convergence}}]\label{t:framework}
Let $(\Omega, \cal F, \Pj)$ be a probability space and $\{\cal F_n\}$ be a filtration of the sigma field $\cal F$. Suppose $X_n\sim p_n$, $Z_n\sim q_n$, and $\ol Z_n\sim \ol q_n$ are $\cal F_n$-adapted random processes taking values in $\Om$, and $B_n\subeq \Om$ are sets such that the following hold for every $n\in \N_0$.
\begin{enumerate}
    \item If $Z_k \in B_k^c$ for all $0\le k\le n-1$, then $Z_n=\ol Z_n$.
    \label{i:couple}
    \item $\chi^2(\ol q_n||p_n)\le D_n^2$.
    \item 
    $\Pj(X_n\in B_n)\le \de_n$.
\end{enumerate}
Then the following hold.
\begin{align}
\TV(q_n, \ol q_n)&\le \sumz k{n-1} (D_k^2+1)^{1/2}\de_k^{1/2}&
\TV(p_n, q_n) &\le D_n+ \sumz k{n-1} (D_k^2+1)^{1/2}\de_k^{1/2} 
\end{align}
% \begin{enumerate}
%     \item $\TV(p_n, \ol q_n)\le \sumz k{n-1} (D_k^2+1)^{1/2}\de_k$
%     \item  
% \end{enumerate}
\end{thm}

\begin{thm}[DDPM with $L^2$-accurate score estimate]
\label{t:ddpm-l2}
Let $0<\ep_\chi,\etv,\de<\rc 2$. 
Suppose that Assumption~\ref{a:tail} for a sufficiently small value of $c$ that $R_0$ is such that $R\pf{c\etv^3\de^6\echi^{12}}{R_0^{19}d^5}\le R_0$, and $R_0^2\ge d$. 
Suppose one of the following cases holds.
\begin{enumerate}
    \item Let $\pdata, s(\cdot, t)$ be such that Assumption~\ref{a:score} holds, %and~\ref{a:bd} hold, 
    with $R_0^2\ge d$. Suppose that
\[
\ep_\si  = O\pa{\fc{\etv \de^{5/2}\echi^{11/2}}{B^{9/4}}%\wedge \fc{\etv \echi^3}{T^{5/2}B}
},
\]
where $B=R_0^4d\ln \pf{T}{\de}\ln\pf{R_0d}{\de\etv\echi}$, 
and we 
run~\eqref{e:ei1} starting from $\ppr$ for time 
$T=\ln \pf{16R_0^2}{\echi^2}$, %\Te\pa{\ln(\CLS d) \vee \CLS \ln\prc{\etv}}$
%\pa{\CLS \ln \pf{2K_\chi}{\ep_\chi^2}}$
$N=
 O\pf{B\pa{T+\rc{\de^2}}}{\echi^2}$ steps with step sizes satisfying $h_k=O\pf{\echi^2}{B\max\{T-t_k, (T-t_k)^{-3}\}}$.
%and step size %$h=\Te\pf{\ep_\chi^2}{\CLS(\CLS+d)(\smf\vee \sms)^2}$ 
\item Let $\pdata, s(\cdot, t)$ be such that Assumptions~\ref{a:score} %,~\ref{a:bd}, 
and~\ref{a:smoothness} hold, with $C\ge R_0^2$. Suppose
\begin{align*}
    \ep_\si  &=  O\pa{
    %\fc{\etv \echi^2}{B^{1/2}}
    %\wedge 
    \fc{\etv \echi^3}{T^{5/2}B}},
\end{align*}
where $B=C^2d\ln\pf T\de \ln \pf{R_0d}{\de\etv\echi}$, and we run~\eqref{e:ei1} starting from $\ppr$ for time 
$T=\ln \pf{16R_0^2}{\echi^2}$, %\Te\pa{\ln(\CLS d) \vee \CLS \ln\prc{\etv}}$
%\pa{\CLS \ln \pf{2K_\chi}{\ep_\chi^2}}$
$N=
 O\pf{B\pa{T+\ln \prc\de}}{\echi^2}$ steps with step sizes satisfying $h_k=O\pf{\echi^2}{B\max\{T-t_k, (T-t_k)^{-1}\}}$.
\end{enumerate}
Then the resulting distribution $q_{t_N}$ is such that $q_{t_N}$ is $\etv$-far in TV distance from a distribution $\ol q_{t_N}$, where $\ol q_{t_N}$ satisfies $\chi^2(\ol q_{t_N}||  p_{t_N} )\le\echi^2$.
 In particular, taking $\echi=\etv$, we have $\TV(q_T,\pdata)\le 2\etv$. 
\end{thm}
Note that the condition on $R$ can be satisfied if $R(\ep) = o(R^{-1/19})$ (no effort has been made to optimize the exponent). 

\begin{proof}%[Proof of Theorem~\ref{t:p-precise}]
We invoke Lemma~\ref{l:KL-covering} for a $\ep_K$ to be chosen, to obtain a distribution $\wt P_0$ on $B_{R_0}(0)$, where $R_0\ge R(\ep_K/8)$. 
Let $B=R_0^4d\ln\pf{T}{\de}\ln\pf{R_0}{\de\ep_K}$ and $B=C^2d\ln \pf{T}{\de}\ln \pf{R_0}{\de\ep_K}$ in case 1 and case 2, respectively; our choice of $\ep_K = O\pf{\etv^2\de^6}{n^2R_0^6}$ will give the definition of $B$ in the theorem statement. In the following, we define $\wt p_t$ with $\wt P_0$, rather than $\pdata$, as the initial distribution. Note that since $\TV(\pdata, \wt P_0)\le \sqrt{\ep_K}=o(\etv)$ (and the same holds for their evolutions under~\eqref{eq:forward_sde}), it suffices to consider convergence to $\wt p_\de$.

We first define the bad sets where the error in the score estimate is large,
\begin{align}
B_t:&=\bc{\ve{\nb \ln \wt p_{t}(x)-s(x, t)}>\ep_{\iy,t}}
\end{align}
for some $\ep_{\iy,t}$ to be chosen. 

Given $t\ge 0$, let $t_- = t_k$ where $k$ is such that $t\in [t_k,t_{k+1})$. Given bad sets $B_t$, define the interpolated process on $[t_k,t_{k+1})$ by
\begin{align}
\label{e:interp-c}
    d\ol z_t &= 
    %-\ba{
    %    f(z_{t_-}, T-t) - g(T-t)^2 b(z_{kh}, T-kh)
    %}\,dt + g(T-t)\,dw_t,\\
    g(T-t)^2 \pa{\rc 2 z_t + b(z_-, T-t_-)}\,dt + g(T-t)\,dw_t\\
    \nonumber
    \text{where }
    b(z,t) &= \begin{cases}
    s(z,t),& z\nin B_t\\
    \nb \ln \wt p_t(z), &z\in B_t
    \end{cases}.
\end{align}
In other words, simulate the reverse SDE using the score estimate as long as the point is in the
%asymp good
good set at the previous discretization timepoint $t_k$, and otherwise use the actual gradient $\nb \ln p_t$. Let $\ol q_t$ denote the distribution of $\ol z_t$ when $\ol z_0\sim q_0$. Note that this process is defined only for purposes of analysis, as we do not have access to $\nb \ln p_t$. As before, we let denote $q_t$ the distribution of $z_t$ defined by~\eqref{e:cts-time}.

We can couple this process with the exponential integrator~\eqref{e:ei1} using $s$ so that as long as $x_{t_m}\nin B_{T-t_m}$, the processes agree, thus satisfying condition~\ref{i:couple} of Theorem~\ref{t:framework}.

% \hlnote{Will need to incorporate error bound from Lemma~\ref{lem:truncate_target}. Also be consistent with $t,T-t$.}
% Then by Chebyshev's inequality,
Then by Lemma~\ref{lem:truncate_target},
\begin{align*}
    \wt P_t^{(0)} (B_t) &=
    \ep_K + 
    \fc{4}{\ep_{\iy, t}^2} \pa{\ep_t^2 +O\pf{\ep_K\pa{d+\ln\prc{\ep_K}}}{\si_t^2} },
    %\pf{\ep_{T-t}}{\ep_{\iy,T-t}}^2=:\de_{T-t}.
\end{align*}
%when $\nb \ln \wt p_t^{(0)}$ is $L_t$-Lipschitz.
%Let $K_\chi = \chi^2(q_0||p_0)$.
Then by choice of $h_k$ and either Corollary~\ref{c:liy-bdd} or~\ref{c:liy-smooth}, %for appropriate choice of constants,
when $\int_0^{t_n}\ep_t^2\,dt=O(1)$, 
\begin{align}
\label{e:chi2-3terms}
\chi^2(\ol q_{t_k}||p_{t_k}) &= e^\ep \chi^2(q_0||p_0) + 
\ep + e^\ep \int_0^{t_n}\ep_{\iy,T-t}^2\,dt\\
&\le 2\chi^2(q_0||p_0) + O(1),
\nonumber 
\end{align}
where $\ep =\fc{\echi^2}{4}$. 
%\hlnote{Optional: work out with dependences on $L,L_s$.}
For $\chi^2(\ol q_{t_k}||p_{t_k})$ to be bounded by $\ep_\chi^2$, it suffices for the terms in~\eqref{e:chi2-3terms} to be bounded by  $\fc{\ep_\chi^2}2, \fc{\ep_\chi^2}4, \fc{\ep_\chi^2}4$; this is implied by
\begin{align}
\nonumber 
    T &%\ge
    =\ln \pf{16R^2}{\echi^2} %=:T_{\min}
    \text{ by Lemma~\ref{l:K-chi}}
    \\
    %\nonumber
%    \ep &\le \fc{\echi^2}4\\
    % h_k &= O\pf{\echi^2}{B\max\{T-t_k, (T-t_k)^{-\al}\}}\\ %O\pf{\ep_\chi^2}{\CLS(\CLS+d)(\smf\vee \sms)^2}\\
    \int_0^{t_n}\ep_{\iy,T-t}^2\,dt &= O(\echi^2).
    \label{e:error-int}
\end{align}
% where $B=R^4d\ln \pf{T}{\de}\ln\pf{R}{\de\ep_K}$  or $\ln \pf{T}{\de}C^2d\ln\pa{\fc{R}{\de\ep_K}}$ in case 1 and case 2, respectively, and $\al=3$, $1$ in case 1 and case 2, respectively.
%\hlnote{for $\echi$ small enough.}

%(We choose $h$ so that the condition in Theorem~\ref{t:p-iy-simple} is satisfied; note $\ep_\chi\le 1$.)
By Theorem~\ref{t:framework}, %\hlnote{+1 should be +2 here}
\begin{align}
\nonumber
\TV(q_{t_n},\ol q_{t_n}) 
&\le 
\sum_{k=0}^{n-1} (1+\chi^2(q_{t_k}||p_{t_k}))^{1/2}P(B_{t_k})^{1/2}\\
&\le 
\sumz k{n-1}\pa{ 2\chi^2(q_0||p_0)^{1/2} + O(1)} \de_t^{1/2}%\\
%&
%\le 
%\fc{\ep}{\ep_1}\pa{2n \chi^2(q_0||p_0) + O(n)}
%= O\pf{\ep n}{\ep_1}.
\\
&= O\pa{\sumz k{n-1}\fc{\ep_{t_k}}{\ep_{\iy,t_k}} + \sqrt{\ep_K} \pa{1+\fc{\sqrt{d+\ln(1/\ep_K)}}{\ep_{\iy, t_k} \si_{T-t_k}}}}.
%\\
% &\le 
% \fc{\ep}{\ep_1}\pa{\fc{64\CLS}hK_\chi + O(n)}.
\label{e:error-sum0}
\end{align}
For this to be bounded by $\etv$, it suffices for 
\begin{align}
\label{e:error-sum}
    \sumz k{n-1} \fc{\ep_t}{\ep_{\iy,t}} &= O(\etv) \\
\label{e:error-sum2}
    %\fc{n\sqrt{\ep_K(d+\ln(1/\ep_K))}}{\ep_{\iy,t}\si_\de}&= O(\etv)
    \ep_K &= O\pf{\min_k \ep_{t_k}^2\si_{T-t_k}^2}{d+\ln(1/\ep_K)}.
\end{align}
We bound~\eqref{e:error-sum2} crudely, as the dependence on $\ep_K$ will be logarithmic. Using %$L_\de = \fc{R^2}{\de^2}$ and 
$\ep_{t_k}^2 = \ep_{\si}^2/\si_{t_k}^4$,
%$\si_\de^2 = \Om(\de)$, 
it suffices that 
\begin{align}
\label{e:error-ep-K}
    \ep_K &= O\pf{\ep_\si^2 }{d+\ln(1/\ep_K)}.
\end{align}
We will return to this after deriving a condition on $\ep_\si$.
It remains to bound~\eqref{e:error-int}  and~\eqref{e:error-sum}. We break up the timepoints depending on whether $T-t>1$. Let 
\begin{align*}
    (t_0,t_1,\ldots, t_N) &= 
    (t_0, \ldots, t_{\ncrs-1},t_0',\ldots, t_{\nfn}')
\end{align*}
and $u_k=T-t_k'$, where $t_{\ncrs-1} \le T-1\le t_1'$. Let $h_k'=t_{k+1}'-t_k'$.
Note the ``fine'' timepoints will be closer together than the ``coarse'' timepoints. 
We break up the integral~\eqref{e:error-int} and the sum~\eqref{e:error-sum} into the parts involving the coarse and fine timepoints. For~\eqref{e:error-int}, it suffices to have
\begin{align*}
    \eqref{e:error-int}, \text{ coarse:}\quad 
    \int_0^{t_0'} \ep_{\iy,T-t}^2\,dt &\le T\max_{0\le k\le \ncrs} \ep_{\iy,T-t_k}^2 = O(\echi^2)
\end{align*}
so it suffices to take $\ep_{\iy,T-t_k}^2 \asymp \fc{\echi^2}{T}$.
Let $\al=3$ in case 1 and $\al=1$ in case 2.
For the fine part, recalling our choice of $h_k'$, it suffices to have (note we can redefine $\ep_t=\ep_{t_k}$ when $t\in [t_k,t_{k+1})$ without any harm)
\begin{align}
\nonumber
\eqref{e:error-int}, \text{ fine:}\quad 
    \int_{t_0'}^{t_{\nfn}'} \ep_{\iy,T-t}^2 \,dt 
    &= 
    \sumz k{\nfn-1} h_k' \ep_{\iy,T-t_k'}^2 = O(\echi^2)\\
    \nonumber 
    \Longleftarrow 
    \sumz k{\nfn-1} \fc{\echi^2 u_k^\al}B \ep_{\iy,u_k}^2& = O(\echi^2)\\
    \label{e:1fine}
    \Longleftarrow
    \sumz k{\nfn-1} \fc{u_k^\al \ep_{\iy,u_k}^2}{B} &= O(1).
\end{align}
For~\eqref{e:error-sum}, it suffices to have
\begin{align}
\nonumber 
\eqref{e:error-sum}, \text{ coarse:}\quad 
\sumz k{\ncrs-1} \fc{\ep_{T-t_k}}{\ep_{\iy,T-t_k}} 
&\asymp \ncrs \fc{\ep_\si }{\echi/\sqrt T} = O(\etv)\\
\Longleftarrow \ep_\si  &= O\pf{\etv \echi}{\ncrs \sqrt T}
\label{e:2coarse}
\end{align}
and 
\begin{align}
\eqref{e:error-sum}, \text{ fine:}\quad 
\sumz k{\nfn-1} \fc{\ep_{u_k}}{\ep_{\iy,u_k}} \asymp \sumz k{\nfn-1} \fc{\ep_\si }{u_k \ep_{\iy,u_k}} = O(\etv). 
\label{e:2fine}
\end{align}
Note that in light of the required step sizes, we can take $\ncrs \asymp \fc{T^2B}{\echi^2}$.
Considering the equality case of H\"older's inequality on $\eqref{e:1fine}^{1/3}\eqref{e:2fine}^{2/3}$ suggests that we take
\begin{align}
\label{e:ep-si}
    \ep_\si &\asymp \fc{\etv B^{1/2}}{\pa{\sumz k{\nfn-1} {u_k}^{\fc{\al-2}3}}^{3/2}}\\
\label{e:ep-iy-uk}
    \ep_{\iy,u_k} &\asymp \fc{B^{1/2}}{{u_k}^{\fc{\al+1}3}\pa{\sumz k{\nfn-1} {u_k}^{\fc{\al-2}3}}^{1/2}}
\end{align}
Note that the number of steps needed in the fine part is $O\pf{B}{\echi^2 \de^{2}}$ in the first case and $O\pf{B}{\echi^2}\ln\prc\de$ in the second case. We can check that~\eqref{e:ep-si} and~\eqref{e:ep-iy-uk} make~\eqref{e:1fine} and~\eqref{e:2fine} satisfied. 

Finally, we calculate the denominator for $\ep_\si $. In case 1, note that starting from $T-t_0'=O(1)$ and taking steps of size $h_k' \asymp \fc{\echi^2}{B(T-t_k')^3}$, it takes $\nfn = \Te\pf{B}{\echi^2\de^2}$ steps to reach $T-t=\de$. 
\begin{align*}
    u_k &= T-t_k' = 
    \pa{1+\Te\pf{k\echi^2}{B}}^{-\rc{2}}\\
    \sumz k{\nfn-1} u_k^{1/3} 
    &\asymp 
\sumz k{\nfn-1} \pa{1+ \Te\pf{k\echi^2}{B}}^{-\fc{1}{6}} \asymp \fc{B}{\echi^2} (\nfn)^{\fc 56}
\asymp  \pf{B}{\echi^2}^{11/6} \rc{\de^{5/3}}.
\end{align*}
Then we obtain
\begin{align*}
    \ep_\si  &\asymp \etv B^{1/2} \fc{\echi^{11/2}}{B^{11/4}} \de^{5/2} = \fc{\etv \de^{5/2}\echi^{11/2}}{B^{9/4}}.
\end{align*}
In case 1, our requirement is
\begin{align*}
    \ep_\si &\asymp O\pa{\fc{\etv \de^{5/2}\echi^{11/2}}{B^{9/4}}\wedge \fc{\etv \echi^3}{T^{5/2}B}},
\end{align*}
but note that the first bound is more stringent.
Now, returning to~\eqref{e:error-ep-K}, we see that it suffices to take $\ep_K =O\pa{\rc d \pf{\etv \de^{5/2}\echi^{11/2}}{R_0^9d^{9/4}}^{2+\be}}$ for any $\be>0$ (this will ``solve'' the $\log(1/\ep_K)$ appearing in $B$.)

In case 2,
we have instead $u_k = \exp\pa{-\Te\pa{\fc{\echi^2}B k}}$ so 
\begin{equation*}
    \ep_\si  \asymp \etv B^{1/2} \pf{\echi^2}B^{3/2} = \fc{\etv \echi^3}{B}. \qedhere
\end{equation*} 
\end{proof}

\begin{thm}[TV error for DDPM with $L^2$-accurate score estimate and smoothness]
\label{t:ddpm-l2-TV}
Let $0<\etv < \rc 2$. 
Suppose that Assumption~\ref{a:smooth0} and ~\ref{a:tail} for a sufficiently small value of $c$ that $R_0$ is such that $R\pf{c\etv^{15}}{R_0^{31}d^5L^{12}}\le R_0$, and $R_0^2\geq \max\bc{d, \E_{P_{\text{data}}}\ba{\ve{X}^2}}$, and one of the following cases holds.
\begin{enumerate}
    \item Let $\pdata, s(\cdot, t)$ be such that Assumption~\ref{a:score} holds. Suppose that
\[
\ep_\si  = O\pa{\fc{\etv^{11.5}}{B^{9/4}R_0^5 L^5}%\wedge \fc{\etv \echi^3}{T^{5/2}B}
},
\]
where $B=R_0^4d\ln \pf{TR_0^2L^2}{\etv^2}\ln\pf{R_0^3dL^2}{\etv^3\echi}$, 
and we 
run~\eqref{e:ei1} starting from $\ppr$ for time 
$T=\ln \pf{16R_0^2}{\etv^2}$, %\Te\pa{\ln(\CLS d) \vee \CLS \ln\prc{\etv}}$
%\pa{\CLS \ln \pf{2K_\chi}{\ep_\chi^2}}$
$N=
 O\pf{B\pa{T+\pf{R_0L}{\etv}^4}}{\etv^2}$ steps with step sizes satisfying $h_k=O\pf{\echi^2}{B\max\{T-t_k, (T-t_k)^{-3}\}}$.
%and step size %$h=\Te\pf{\ep_\chi^2}{\CLS(\CLS+d)(\smf\vee \sms)^2}$ 
\item Let $\pdata, s(\cdot, t)$ be such that Assumptions~\ref{a:score} %,~\ref{a:bd}, 
and~\ref{a:smoothness} hold, with $C\ge R_0^2$. Suppose
\begin{align*}
    \ep_\si  &=  O\pa{
    %\fc{\etv \echi^2}{B^{1/2}}
    %\wedge 
    \fc{\etv^4}{T^{5/2}B}},
\end{align*}
where $B=C^2d\ln\pf{TR_0^2L^2}{\etv^2} \ln \pf{R_0^3dL^2}{\etv^4}$, and we run~\eqref{e:ei1} starting from $\ppr$ for time 
$T=\ln \pf{16R_0^2}{\etv^2}$, %\Te\pa{\ln(\CLS d) \vee \CLS \ln\prc{\etv}}$
%\pa{\CLS \ln \pf{2K_\chi}{\ep_\chi^2}}$
$N=
 O\pf{B\pa{T+%2
 \ln \pf{R_0L}{\etv}}}{\etv^2}$ steps with step sizes satisfying $h_k=O\pf{\echi^2}{B\max\{T-t_k, (T-t_k)^{-1}\}}$.
\end{enumerate}
Then the resulting distribution $q_{t_N}$ is such that $q_{t_N}$ is $\etv$-far in TV distance from the data distribution $P_{\text{data}}$. 
\end{thm}
\begin{proof}
With the result of Theorem~\ref{t:ddpm-l2}, we see that $\TV(q_{t_N}, p_{t_N})\leq 2\etv$. Now by Lemma~\ref{l:bdd_spt}, if we further assume
\begin{align*}
    \delta = O\pf{\ep_{\TV}^2}{R_0^2L^2},
\end{align*}
then $\TV(p_{t_N}, P_{\text{data}}) \leq \ep_{\TV}$. We conclude the proof by triangle inequality and replacing the $\delta$-dependence with $O(\fc{\ep_{\TV}^2}{R_0^2L^2})$ in the previous theorem.
\end{proof}
\begin{proof}[Proof of Theorem~\ref{t:tv-tail}]
If $\pdata$ is subexponential with a fixed constant, note that Assumption~\ref{a:tail} holds with  $R(\ep)=O\pa{\ln\prc\ep}$ and hence $R_0$ is logarithmic in all parameters.
\end{proof}
\subsection{Wasserstein error guarantees}

\begin{proof}[Proof of Theorem~\ref{t:main-tv-w2}]
If $T-t_N=\de$, then $W_2(\wt p_0,\wt p_\de) \le \si_\de \le \sqrt \de$. Choosing $\de =\ew^2$, we see by Theorem~\ref{t:ddpm-l2} it suffices to take 
\begin{align*}
    \ep_\si &= O\pa{\fc{\etv^{13/2} (\ew^2)^{5/2}}{\pa{R^4d\ln \pf{T}{\de}\ln\pa{\fc{RN}{\de\ew}}}^{9/4}}}.
\end{align*}
Simplifying gives $\ep_\si = \wt o\pf{\etv^{6.5} \ew^5}{R^9d^{2.25}}$.
If Assumption~\ref{a:smoothness} also holds, then it suffices to take
\begin{align*}
    \ep_\si &= 
    O\pf{\etv^4}{T^{5/2}C^2d\ln\pf T\de \ln \pf{RN}{\de\ew}}.
\end{align*}
Simplifying gives 
$\ep_\si = \wt o\pf{\etv^4}{C^2 d}$.
\end{proof}

\begin{proof}[Proof of Theorem~\ref{t:main-w2}]
To obtain purely Wasserstein error guarantees, we include an extra step of replacing any sample $z_{t_N}\sim q_{t_N}$ falling outside $B_R(0)$ by $0$. 
Suppose $T-t_N=\de$. 
Let $\wh q_{t_N}$ be the resulting distribution. Then
\begin{align*}
    W_2(\wt p_0, \wh q_{t_N})
    &\le 
    W_2(\wt p_0, \wt p_{\de}) + W_2(\wt p_{\de}, \wh q_{t_N}) \\
    &\le % \si_{T-t_N} 
    \si_{\de} + 
    %O(\de^{1/2})+
    W_2(\wt p_{\de}, \wt q_{t_N}) \le \sqrt{\de} + 
    W_2(\wt p_{\de}, \wh q_{t_N}).
    %2R
    %\TV(\wt p_{\de}, \wt q_{t_N})
    %\ep_{\TV}
\end{align*}
We choose $\de=\fc{\ew^2}4$ so the first term is $\le \fc{\ew}2$. It suffices to bound the second term  $W_2(\wt p_{\de}, \wh q_{t_N})$ also by $\fc{\ew}2$. We bound it in terms of  $\TV(\wt p_{\de}, \wh q_{t_N})$ using the fact that $\wh q_{t_N}$ is supported on $B_R(0)$ and using a Gaussian tail calculation for $\wt p_\de$. Consider a coupling of $x_{t_N} = \wt x_{\de}\sim\wt p_\de$ and $\wh z_{t_N}\sim \wh q_{t_N}$ such that $x_{\de}\ne \wh z_{t_N}$ with probability $\etv$. Express $\wt x_\de = m_{\de} \wt x_0 + \si_\de \xi$ where $\wt x_0\sim  \wt p_0$. Now
\begin{align*}
    \E[\ve{\wt x_{\de} - \wh z_{t_N}}^2]
    &\le
    \sup_{P(A)\le \etv}
    2\pa{\E[\ve{m_\de \wt x_0 -z_{t_N} }^2\one_A]  + \si_\de^2\E[\ve{ \xi}^2\one_A]}\\
    &= 2\pa{4R^2\etv + \si_\de^2\etv  \cdot O\pa{d+\ln\prc{\etv}}},
\end{align*}
where the bound on the second term uses Lemma~\ref{l:gtail}.
Using $R^2\ge d$, we see that it suffices to choose $\etv = O\pf{\ew^2}{R^2}$ for appropriate choice of constants.
% Hence
% \begin{align*}
%     W_2(\wt p_\de, \wh q_{t_N})
%     &=
%     O(R\sqrt{\etv}).
%     %\le R\sqrt{14\etv}.
% \end{align*}
% It suffices to choose $\de = O\pf{\ew^2}{d+\ln\prc{\etv}}$
% and $\etv = O\pf{\ew^2}{R^2}$ for appropriate choice of constants. 
By Theorem~\ref{t:ddpm-l2}, it suffices to take 
\begin{align*}
    \ep_\si &= O\pa{\fc{(\ew^2/R^2)^{13/2} \pa{\ew^2}^{5/2}}{\pa{R^4d\ln \pf{T}{\de}\ln\pa{\fc{RN}{\de\ew}}}^{9/4}}}.
\end{align*}
Simplifying gives $\wt o\pf{\ew^{18}}{R^{22}d^{2.25}}$. 

In case 2, it suffices to take
\begin{align*}
    \ep_\si &= 
    O\pf{(\ew^2/R^2)^4}{T^{5/2}(C^2d\ln\pf T\de \ln \pf{RN}{\de\ew})}.
\end{align*}
Simplifying gives $\ep_\si= \wt o\pf{\ew^8}{C^2R^8 d}$.
%Choosing $T-t_N\le \fc{\ew^{-2}}4$ and $\ep_{\TV}\le \fc{\ew}{2R}$ then gives $W_2(p_0, \wt q_{T-t_N})\le \ew$.
%If $\wt p_0$
\end{proof}

\printbibliography

%%%%%%%%%%%%%%%%%%%%%%%%%%%%%%%%%%%%%%%%%%%%%%%%%%%%%%%%%%%%

%%%%%%%%%%%%%%%%%%%%%%%%%%%%%%%%%%%%%%%%%%%%%%%%%%%%%%%%%%%%

\appendix

%\section{Improved bound using a high-probability bound on the Hessian}
\section{High-probability bound on the Hessian}
\label{s:hess-whp}
%Note that our arguments rely only on good behavior of the backwards DDPM process on a subset of high probability. Hence, if we can obtain an improved bound for the Lipschitz constant of the score function on a set of high probability, then we can let the ``bad set'' $B_t$ include the event that this Lipschitz constant is large, and obtain a better bound for our main theorem.
In this section we obtain a high-probability bound on the Hessian of $\ln \wt p_t$, i.e., the Jacobian of the score function.

To see why we expect Hessian to usually be smaller than the worst-case bound given by Lemma~\ref{l:Hess-bd}, note that we can express~\eqref{e:grad-ln} and~\eqref{e:grad2-ln} as
\begin{align}
\label{e:Y-X} 
    \nb \ln(\mu*\ph_{\si^2}(y)) 
    &= -\rc{\si^2} \E[Y-X|Y=y]\\
\label{e:Y-X2} 
    \nb^2 \ln(\mu*\ph_{\si^2}(y)) 
    &= \rc{\si^4} \Cov[Y-X|Y=y] - \rc{\si^2} I_d
\end{align}
where $X\sim \mu$ and $Y=X+\si\xi$, $\xi\sim N(0,I_d)$. We expect that the random variable $Y-X$ is distributed as $N(0,\si^2I_d)$, which suggests
that the covariance~\eqref{e:Y-X2} may be bounded by $\rc{\si^2}$ rather than $\rc{\si}$ with high probability. Indeed, we can easily construct an example where the worst case of Lemma~\ref{l:Hess-bd} is attained---for example, $\mu=\rc2 (\de_{-v}+\de_v)$ for $\ve{v}_2=R$, at $x=0$---but this point has exponentially small probability density under $\mu*\ph_{\si^2}$.

The following lemma uses a $\ep$-net argument to bound the operator norm of the variance of a conditional distribution, with high probability.
\begin{lem}
\label{l:cov-whp}
Suppose $X$ is a $\R^d$-valued random variable over the probability space $(\Om, \mathcal G,P)$, and $\mathcal F\subeq \mathcal G$ is a $\si$-subalgebra. If $X$ is subgaussian, then
\begin{align*}
    \Pj\pa{
        \E\ba{\ve{XX^\top}|\cal F}
        \ge 2\ve{X}_{\psi_2}^2 \ln\pf{2\cdot 5^d}{\ep}
    }
    &\le \ep.
\end{align*}
\end{lem}
\begin{proof}
    By Jensen's inequality and Markov's inequality, for any $v\in \bS^{d-1}$,
    \begin{align*}
        \Pj\pa{\E[v^\top XX^\top v|\cal F] \ge \la^2 }
        &= 
        \Pj\pa{e^{\E[v^\top XX^\top v|\cal F]/c^2} \ge e^{\la^2/c^2} }\\
        &\le \Pj\pa{
            \E\ba{e^{\an{X,v}^2/c^2}|\cal F} \ge e^{\la^2/c^2}
        }\\
        &\le 
        \fc{\E\ba{\E[ e^{\an{X,v}^2/c^2}|\cal F]}}{e^{\la^2/c^2}}
        = \fc{\E\ba{ e^{\an{X,v}^2/c^2}}}{e^{\la^2/c^2}} \le 2e^{-\la^2/\ve{X}_{\psi_2}},
    \end{align*}
    where the last inequality follows from taking $c=\ve{X}_{\psi_2}$. Now take a $\rc2$-net $\cal N$ of $\bS^{d-1}$ of size $\le 5^d$ \cite[Cor. 4.2.13]{vershynin2018high}. By a union bound,
    \begin{align*}
        \Pj\pa{\exists v\in \cal N: \E[v^\top XX^\top v|\cal F]\ge \la^2} &\le 5^d \cdot 2 \cdot e^{-\la^2/\ve{X}_{\psi_2}^2}=\ep
    \end{align*}
    when we take $\la = \ve{X}_{\psi_2} \sqrt{\ln \pf{2\cdot 5^d}{\ep}}$.
    By \cite[Lemma 4.4.1]{vershynin2018high}, the operator norm can be bounded by the norm on an $\ep$-net,
    \begin{align*}
        \ve{A} &\le 
        2\sup_{v\in A} \ve{\an{A,v}}
        = 2\sup_{v\in A}|v^\top Av|.
    \end{align*}
    where the second inequality holds when $A$ is symmetric.
    The result follows from applying this to $\E[v^\top XX^\top v|\cal F]$.
\end{proof}
From this we obtain the desired high-probability bound.
\begin{lem}
\label{l:Hess-bd-whp}
There is a universal constant $C$ such that the following holds.
For any starting distribution $\wt P_0$, letting $\wt P_t$ be the law of the DDPM process~\eqref{eq:forward_sde} at time $t$,
    we have
    \begin{align*}
        \wt P_t\pa{\ve{\nb^2 \ln \wt p_t(x)} \le \fc{C(d+\ln \prc\ep)}{\si_t^2}}
        &\ge 1-\ep. 
    \end{align*}
\end{lem}
Note that there is no dependence on the radius.
%\hlnote{But do we need any kind of moment assumption???}
\begin{proof}
    Apply~\eqref{e:Y-X2} with $\mu = M_{m_t\sharp}\wt P_0$ to obtain $\nb^2\ln \wt p_t$.
    Noting that $Y-X\sim N(0,\si^2I_d)$ is subgaussian with $\ve{Y-X}_{\psi_2}\le C_2\si$ for some universal constant $C_2$, the result follows from Lemma~\ref{l:cov-whp}.
\end{proof}

\end{document}